%% file: preprint.tex
\documentclass{article} 
\usepackage{iclr2026_conference,times}

\input{math_commands.tex}

\usepackage{hyperref}
\usepackage{url}
\usepackage{booktabs}
\usepackage{amsmath,amssymb}
\usepackage{mathtools}
\usepackage{amsthm}
\usepackage{graphicx}
\usepackage{tikz}
\usetikzlibrary{arrows.meta,positioning}
\usepackage{quantikz}

\providecommand{\ket}[1]{\lvert#1\rangle}
\providecommand{\bra}[1]{\langle#1\rvert}

\theoremstyle{plain}
\newtheorem{proposition}{Proposition}
\newtheorem{lemma}{Lemma}
\theoremstyle{definition}
\newtheorem{definition}{Definition}

\title{Effective Dimension Governs Generalization in Quantum Kernel Vision Models}

\author{Jian Xu$^{1,2}$, \; Delu Zeng$^{3}$, \; John Paisley$^{4}$, \; Qibin Zhao$^{2}$ \\[2pt]
$^1$RIKEN iTHEMS \quad $^2$RIKEN AIP \quad $^3$South China University of Technology \quad $^4$Columbia University \\[2pt]
\texttt{jian.xu@riken.jp}}

\newcommand{\deff}{d_{\mathrm{eff}}}

\iclrfinalcopy
\begin{document}

\maketitle
\lhead{}\rhead{}\chead{}

\begin{abstract}
Recent quantum vision models---quantum vision transformers and quantum
convolutional networks---report two striking but unexplained empirical
phenomena: (i) ansatze with more, or more uniformly distributed, entanglement
generalize better, and (ii) injecting quantum noise can \emph{improve} test
accuracy rather than degrade it. These observations are currently treated as
curiosities, discovered by grid search and explained, if at all, by hand. We
show that both are manifestations of a single, measurable quantity: the
\emph{effective dimension} $\deff$ of the (noise-shaped) quantum feature
kernel. Working primarily with quantum-kernel vision models---a quantum feature
map read out by a kernel classifier---we give a spectral account in which
entanglement structure and quantum noise are two knobs that move $\deff$; in an
overfitting regime, contracting $\deff$ acts as ridge-like regularization. We
analyze the mechanism: an \emph{exact} decomposition of the depolarized kernel
$K_p=(1-p)^2K+\tfrac{p(2-p)}{D}\mathbf{1}\mathbf{1}^\top$ with $\deff(K_p)\to1$, a
contraction result (and its boundary) for amplitude damping, a kernel-machine
capacity bound, and a capacity/alignment risk decomposition; the monotone
contraction operative in our entangled experiments is verified empirically, not
proven in general. Our informative empirical finding is that test accuracy
collapses onto a single function of $\deff$ \emph{across distinct entangling
ansatze}---compressing different spectral shapes onto one curve
($R^2{=}0.82\pm0.08$ over seeds). Along the one-parameter depolarizing family the
collapse is instead exact \emph{by construction}; we use it only to confirm the
kernel decomposition to machine precision and at up to $12$ qubits, not as
evidence for $\deff$. Amplitude damping contracts $\deff$ and lifts test accuracy
by up to $+13\%$ along an inverted-U sweet spot; the effect's sign flips between
the over- and under-fitting regimes; noise injection matches an explicit
spectral-filtering frontier (so it is not a weak substitute for hyperparameter
tuning); and the phenomenon persists in \emph{trained} QViT- and QCNN-like models.
Entanglement plays the complementary role of a precondition---it supplies the
feature-space alignment without which the $\deff$ law does not hold. Our results
organize two reported anecdotes into a single measurable principle for designing
quantum-vision models.
\end{abstract}

\section{Introduction}

Quantum machine learning for computer vision has advanced rapidly, with
quantum vision transformers (QViTs) \citep{cherrat2024quantum, boucher2025n,
zhang2025hqvit} and quantum convolutional neural networks (QCNNs)
\citep{cong2019quantum, shi2024quantum, roseler2025efficient, long2025hybrid,
wu2025hybrid} now matching strong classical baselines on small benchmarks while
using dramatically fewer parameters. Yet the field's design practice remains
largely empirical: ansatze are chosen by sweeping heuristic descriptors such as
expressibility and entangling capability \citep{sim2019expressibility,
roseler2025efficient}, and the role of hardware noise is assessed
\emph{post hoc} through simulation.

Two recurring empirical observations stand out as genuinely puzzling. First,
several works report that \emph{more entanglement helps generalization}: for
example, \citet{wang2025hybrid} find that ``ansatzes with uniformly
distributed entanglement entropy consistently deliver superior non-local
feature fusion and state-of-the-art accuracy.'' Second, and more surprisingly,
the same work reports that \emph{quantum noise can help}: amplitude damping
improves accuracy by $+2.71\%$ in some configurations, a ``double-edged''
behavior with a non-monotonic dependence on noise strength. Both phenomena are
reported as discoveries, without a predictive theory of \emph{when} they occur
or \emph{how strong} the effect should be.

\paragraph{This paper.} We argue that these two phenomena are not separate, and
not mysterious. They are two views of the same underlying object---the
eigenspectrum of the quantum feature kernel---and in particular of its
\emph{effective dimension}
\begin{equation}
\deff \;=\; \frac{\big(\sum_i \lambda_i\big)^2}{\sum_i \lambda_i^2}
\end{equation}
the participation ratio of the kernel eigenvalues $\{\lambda_i\}$
(defined formally in Section~\ref{sec:theory}). Our thesis is:

\begin{quote}
\emph{Entanglement structure and quantum noise are two knobs that move $\deff$.
Within the entangled regime, generalization is governed by $\deff$ alone; under
overfitting, contracting $\deff$ acts as regularization and improves
generalization up to an interior optimum.}
\end{quote}

Under this view, entangling gates redistribute amplitude across the Hilbert
space and noise channels contract the spectrum; both regulate $\deff$, and the
``noise helps'' effect is spectral (ridge-like) regularization in an overfitting
regime. Section~\ref{sec:theory} makes the tractable parts precise---global
depolarizing noise contracts the spectrum \emph{exactly} with $\deff\to1$
(Prop.~\ref{prop:noise}), amplitude damping contracts it with a stated boundary
(Prop.~\ref{prop:ad}), capacity is bounded by $\deff$ (Prop.~\ref{prop:cap}), and
risk splits into a spectral and an alignment term (Prop.~\ref{prop:align})---while
being explicit about the limits: strict monotonicity in the entangled,
amplitude-damped regime that we actually run holds only under a constant-row-sum
condition (which Perron--Frobenius does \emph{not} supply) and is otherwise an
empirical observation. This reframing has three consequences, which we state as
predictions and verify experimentally:
\begin{enumerate}
\item \textbf{(P1) $\deff$ organizes generalization.} Across distinct entangling
ansatze (different spectral shapes), test accuracy collapses onto a single,
stable function of $\deff$ ($R^2=0.82\pm0.08$ across seeds), and the \emph{sign} of
the dependence flips between the overfitting and underfitting regimes---consistent
with the bias--variance picture of Prop.~\ref{prop:cap} (which fixes the sign,
though not the precise location of the interior optimum).
\item \textbf{(P2) Noise is a spectral regularizer.} Increasing quantum noise
monotonically contracts $\deff$, reduces the train--test gap, and---when the
model overfits---improves test accuracy along an inverted-U sweet spot,
reproducing the reported ``noise helps'' phenomenon.
\item \textbf{(P3) Entanglement is a precondition.} Entanglement supplies the
feature-space alignment (Prop.~\ref{prop:align}) that places a model in the
regime where the $\deff$ law holds: entangled circuits of differing topology all
lie on the same accuracy--$\deff$ curve, whereas the unentangled (product) map
sits off it---not through $\deff$, but through its low label alignment.
\end{enumerate}

Beyond explaining existing observations, P3 implies an actionable design view:
to improve a quantum vision model one should monitor and steer a single,
cheaply measured spectral quantity rather than tune entanglement and noise as
independent heuristics.

\paragraph{Contributions.}
\begin{itemize}
\item A \textbf{spectral theory} (Section~\ref{sec:theory}). We give the exact
depolarized-kernel decomposition with $\deff(K_p)\to1$ and an explicit formula for
$\deff(K_p)$ (Prop.~\ref{prop:noise}); a contraction result for amplitude damping
with its analytic boundary (Prop.~\ref{prop:ad}); a bridge showing $\deff$ and the
ridge dimension $d_\gamma$ co-vary (Lemma~\ref{lem:bridge}); a capacity bound
(Prop.~\ref{prop:cap}); and a capacity/alignment decomposition under which the
accuracy--$\deff$ collapse is \emph{exact} along the one-parameter noise family and
conditional across ansatze (Prop.~\ref{prop:align}). We are careful about what is
proven: global monotonicity of $\deff(K_p)$ holds under constant row sums and
otherwise empirically, and the cross-ansatz collapse is an empirical claim.
\item \textbf{Empirical verification} (mean$\pm$std over $5$ seeds). Our central
empirical result is that test accuracy collapses onto a single $\deff$ curve
\emph{across distinct entangling ansatze} ($R^2{=}0.82\pm0.08$), compressing
different spectral shapes onto one curve. Amplitude damping reproduces the ``noise
helps'' effect with an inverted-U sweet spot; the $\deff$--accuracy sign flips
between regimes; noise injection lands on an explicit spectral-filtering frontier
(not a weak substitute for tuning); and entanglement enters as an alignment
precondition (measured $\mathrm A(K)$). Separately, we verify the exact kernel
decomposition to machine precision and at up to \textbf{12 qubits}; the perfect
collapse of that one-parameter family is a by-construction null, reported as a
scale check rather than as evidence for $\deff$.
\item \textbf{Realistic settings, robustness, and honest scope}. The mechanism
persists in \emph{trained} QViT- and QCNN-like models; it replicates on
Fashion-MNIST and a medical task (BloodMNIST) and across depth/width; and we mark
its boundaries---$\deff$ is dominant but not a universal sufficient statistic,
training adapts $\deff$ on its own. We also \emph{validate on a real} IBM Heron
device (\texttt{ibm\_kawasaki}): in an overfitting regime the intrinsic hardware
noise contracts $\deff$ ($4.06\!\to\!3.38$) and \emph{improves} test accuracy
($0.863\!\to\!0.900$), realizing the regularization mechanism on silicon.
\end{itemize}

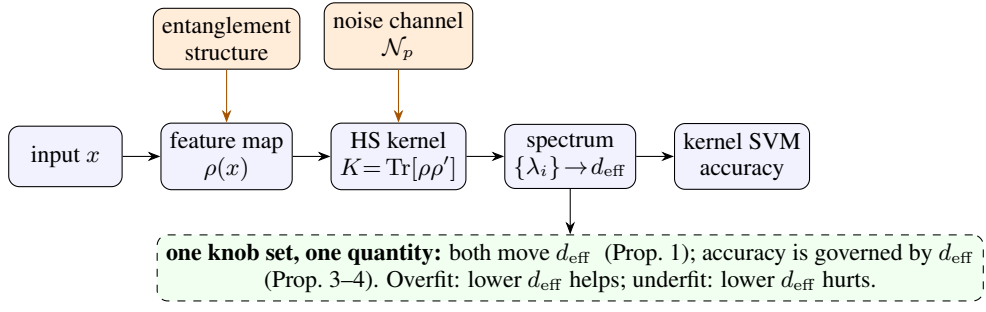
\begin{figure}[t]
\centering
\begin{tikzpicture}[font=\footnotesize,>=Stealth,node distance=5mm,
  box/.style={draw,rounded corners,minimum height=8mm,minimum width=15mm,align=center,fill=blue!5},
  knob/.style={draw,rounded corners,minimum height=6.5mm,minimum width=17mm,align=center,fill=orange!15},
  reg/.style={draw,dashed,rounded corners,align=center,fill=green!6}]
\node[box] (x) {input $x$};
\node[box,right=of x] (fm) {feature map\\ $\rho(x)$};
\node[box,right=of fm] (k) {HS kernel\\ $K{=}\operatorname{Tr}[\rho\rho']$};
\node[box,right=of k] (d) {spectrum\\ $\{\lambda_i\}\!\to\! \deff$};
\node[box,right=of d] (clf) {kernel SVM\\ accuracy};
\foreach \a/\b in {x/fm,fm/k,k/d,d/clf}{\draw[->] (\a)--(\b);}
\node[knob,above=7mm of fm] (ent) {entanglement\\ structure};
\node[knob,above=7mm of k] (noise) {noise channel\\ $\mathcal N_p$};
\draw[->,orange!65!black] (ent)--(fm);
\draw[->,orange!65!black] (noise)--(k.north);
\node[reg,below=6mm of d,minimum width=70mm] (r)
 {\textbf{one knob set, one quantity:} both move $\deff$
 \;(Prop.~\ref{prop:noise}); accuracy is governed by $\deff$ \\
 (Prop.~\ref{prop:cap}--\ref{prop:align}). Overfit: lower $\deff$ helps;\ underfit: lower $\deff$ hurts.};
\draw[->] (d)--(r);
\end{tikzpicture}
\caption{Overview. A quantum feature map $\rho(x)$ induces a Hilbert--Schmidt
kernel whose spectrum is summarized by the effective dimension $\deff$.
Entanglement topology and an injected noise channel are two knobs that both move
$\deff$; within the entangled regime, generalization is a function of $\deff$
alone, and the sign of its effect is set by the bias--variance regime.}
\label{fig:overview}
\end{figure}

\section{Related work}

\paragraph{Quantum vision models.} QViTs replace classical self-attention with
parameterized quantum circuits, reducing parameter counts from $O(n^2)$ to
$O(n)$ \citep{boucher2025n} or using amplitude encoding to process whole
images with $O(\log N)$ qubits \citep{zhang2025hqvit}; \citet{cherrat2024quantum}
give compound-matrix attention with provable asymptotic advantages.
QCNNs \citep{cong2019quantum} have been extended to multiclass settings
\citep{shi2024quantum}, hardware-efficient encodings \citep{roseler2025efficient},
inception-style heterogeneous filters \citep{wu2025hybrid}, and trainable
quantum--classical--quantum stacks \citep{long2025hybrid}. Across this
literature, ansatz selection relies on expressibility and entangling-capability
heuristics \citep{sim2019expressibility}, and noise is studied empirically.
We provide the missing predictive layer.

\paragraph{Quantum models as kernel methods.} Supervised quantum models with
fixed feature maps are kernel methods \citep{schuld2021supervised,
schuld2019quantum, havlivcek2019supervised}, which lets us analyze generalization
through the kernel spectrum. The effective dimension has been used to
characterize the capacity of quantum neural networks \citep{abbas2021power},
and the structure of the data-induced kernel governs the possibility of quantum
advantage \citep{huang2021power}. Kernel--target alignment \citep{cristianini2001kernel}
is a classical task-aware spectral descriptor. We connect these spectral tools
directly to the entanglement/noise design choices made in quantum vision and,
crucially, make them \emph{predictive} of the optimal configuration.

\section{A spectral theory of entanglement and noise}
\label{sec:theory}

\subsection{Setup and definitions}
We consider hybrid models whose quantum component is a feature map
$x \mapsto \rho(x)\in\mathbb{C}^{D\times D}$, $D=2^{n_q}$, where $\rho(x)$ is the
density matrix produced by a parameterized circuit acting on an encoding of an
input $x$. Following the kernel view of quantum models
\citep{schuld2021supervised, havlivcek2019supervised}, classification is performed
by a kernel machine on the Hilbert--Schmidt (HS) kernel
\begin{equation}
k(x,x') \;=\; \operatorname{Tr}\!\big[\rho(x)\,\rho(x')\big]
\;=\; \langle \phi(x),\phi(x')\rangle_{\mathrm{HS}},\qquad
\phi(x):=\mathrm{vec}\,\rho(x),
\label{eq:kernel}
\end{equation}
which for pure states reduces to the fidelity kernel
$|\langle\psi(x)|\psi(x')\rangle|^2$.

\begin{definition}[Gram matrix and spectral descriptors]
Given training inputs $\{x_i\}_{i=1}^n$, let $K\in\mathbb{R}^{n\times n}$,
$K_{ij}=k(x_i,x_j)$, with eigenvalues $\lambda_1\ge\cdots\ge\lambda_n\ge0$.
Define the \emph{effective dimension} (participation ratio / effective rank)
\begin{equation}
\deff(K)=\frac{\big(\sum_i\lambda_i\big)^2}{\sum_i\lambda_i^2}
=\frac{(\operatorname{tr}K)^2}{\operatorname{tr}(K^2)}\in[1,n],
\label{eq:deff}
\end{equation}
the ridge effective dimension $d_\gamma(K)=\sum_i \lambda_i/(\lambda_i+\gamma)
=\operatorname{tr}\!\big[K(K+\gamma I)^{-1}\big]$, and, for one-hot centered
labels $Y$ with target Gram $T=YY^\top$, the kernel--target alignment
$\mathrm{A}(K)=\langle K,T\rangle_F/(\|K\|_F\|T\|_F)$ \citep{cristianini2001kernel}.
\end{definition}

\begin{lemma}[Validity]
\label{lem:psd}
$k$ in \eqref{eq:kernel} is a positive-semidefinite kernel, so $K\succeq0$ and
$\deff(K)$ is well defined; $\deff(K)=1$ iff $K$ has rank $1$ and $\deff(K)=n$
iff the spectrum is flat.
\end{lemma}
\begin{proof}
$k(x,x')=\langle\phi(x),\phi(x')\rangle_{\mathrm{HS}}$ is an inner product of the
feature vectors $\phi(x)=\mathrm{vec}\,\rho(x)$, hence PSD; the endpoint
characterizations are Cauchy--Schwarz applied to $(\lambda_i)$.
\end{proof}

\subsection{Noise contracts the spectrum (P2)}
We first treat noise analytically. Let $\mathcal{N}_p$ be the global depolarizing
channel of strength $p\in[0,1]$, $\mathcal{N}_p[\rho]=(1-p)\rho+p\,I/D$, applied
to the feature map: $\rho_p(x)=\mathcal{N}_p[\rho(x)]$, with kernel
$K_p$, $(K_p)_{ij}=\operatorname{Tr}[\rho_p(x_i)\rho_p(x_j)]$.

\begin{proposition}[Depolarizing noise is exact spectral ridge filtering]
\label{prop:noise}
With $\mathbf{1}=(1,\dots,1)^\top\in\mathbb{R}^n$,
\begin{equation}
K_p \;=\; (1-p)^2\,K \;+\; \tfrac{p(2-p)}{D}\,\mathbf{1}\mathbf{1}^\top .
\label{eq:Kp}
\end{equation}
Hence $K_p$ interpolates from $K_0=K$ to the rank-one matrix
$\tfrac1D\mathbf1\mathbf1^\top$ as $p\to1$, so $\deff(K_p)\to1$. Writing
$s=(1-p)^2$, $T=\operatorname{tr}K$, $Q=\operatorname{tr}(K^2)$, and
$S=\mathbf1^\top K\mathbf1$, the effective dimension is the explicit ratio
\begin{equation}
\deff(K_p)=\frac{\big(sT+\tfrac{p(2-p)}{D}n\big)^2}
{s^2Q+2s\tfrac{p(2-p)}{D}\,S+\big(\tfrac{p(2-p)}{D}\big)^2 n^2}.
\label{eq:deffp}
\end{equation}
\emph{If $K$ has constant row sums ($K\mathbf1=c\mathbf1$, i.e.\ $\mathbf1$ is an
eigenvector of $K$)} then $K$ and $\mathbf1\mathbf1^\top$ commute and
$\deff(K_p)$ is monotonically non-increasing in $p$. In general the cross-term
$S=\mathbf1^\top K\mathbf1$ prevents a structure-free monotonicity guarantee;
empirically $\deff(K_p)$ is \emph{strictly} decreasing in every experiment we run
(Tables~\ref{tab:noise},~\ref{tab:sweet}, and the exact-depolarizing runs of
Sec.~\ref{sec:robust}).
\end{proposition}
\begin{proof}
Expanding $\operatorname{Tr}[\rho_p(x_i)\rho_p(x_j)]$ using $\operatorname{Tr}\rho=1$
and $\operatorname{Tr}I=D$ gives the four terms
$(1-p)^2K_{ij} + 2\frac{(1-p)p}{D} + \frac{p^2}{D}=(1-p)^2K_{ij}+\frac{p(2-p)}{D}$,
which is \eqref{eq:Kp}; the limit and \eqref{eq:deffp} follow by direct
computation of $\operatorname{tr}K_p$ and $\operatorname{tr}(K_p^2)$ using
$\operatorname{tr}(\mathbf1\mathbf1^\top)=n$,
$\operatorname{tr}(K\mathbf1\mathbf1^\top)=S$, and
$\operatorname{tr}((\mathbf1\mathbf1^\top)^2)=n^2$. When $K\mathbf1=c\mathbf1$, $K$
and $\mathbf1\mathbf1^\top$ are simultaneously diagonalizable with
$\mu_1(p)=s\lambda_1+\frac{p(2-p)}{D}n$ and $\mu_{k\ge2}(p)=s\lambda_k$; the
top-eigenvalue share then increases in $p$ (App.~\ref{app:proofs}), so $\deff$
decreases. Full computation and the obstruction in the general case are in
App.~\ref{app:proofs}.
\end{proof}

Equation \eqref{eq:Kp} makes the mechanism explicit: noise shrinks the
informative component $(1-p)^2K$ while adding a rank-one ``constant'' component,
i.e.\ it is a spectral low-pass (ridge) filter.

\paragraph{Amplitude damping.} The channel actually used in our experiments is
per-qubit amplitude damping $\mathcal A_\gamma$ (rate $\gamma$), which lacks the
exact rank-one decomposition of depolarizing noise because it is \emph{non-unital}.
We nonetheless prove the essential contraction.

\begin{proposition}[Amplitude damping contracts the feature spectrum]
\label{prop:ad}
Let $\rho_\gamma(x)=\mathcal A_\gamma^{\otimes n_q}[\rho(x)]$.
\textbf{(a) Limit.} $\rho_\gamma(x)\!\to\!\ket{0}\!\bra{0}^{\otimes n_q}$ for every
$x$, so $K_\gamma\to\mathbf1\mathbf1^\top$ and $\deff(K_\gamma)\to1$ as
$\gamma\to1$. \textbf{(b) Strict single-qubit contraction.} On one qubit
$\mathcal A_\gamma$ maps the Bloch vector $(X,Y,Z)\mapsto(\sqrt{1-\gamma}\,X,
\sqrt{1-\gamma}\,Y,(1-\gamma)Z+\gamma)$, hence
$\|\mathcal A_\gamma(\rho)-\mathcal A_\gamma(\sigma)\|_{\mathrm{HS}}\le
\sqrt{1-\gamma}\,\|\rho-\sigma\|_{\mathrm{HS}}$; for product feature maps this
tensorizes, so every Gram entry moves toward $1$ and $\deff(K_\gamma)$ is
non-increasing. \textbf{(c) Boundary.} Because $\mathcal A_\gamma$ is non-unital
($\mathcal A_\gamma(I)=I+\gamma Z$), $\mathcal A_\gamma^{\otimes n_q}$ can expand
the Hilbert--Schmidt norm of trace-carrying operators, which is exactly why the
clean global monotonicity of the depolarizing case need not extend to arbitrary
\emph{entangled} ensembles; monotone contraction is guaranteed under the constant-row-sum
condition of Prop.~\ref{prop:noise} and holds in every experiment we run (the
$\deff$ columns of Tables~\ref{tab:noise},~\ref{tab:sweet} decrease strictly).
\end{proposition}

Proof in App.~\ref{app:proofs}; the single-qubit computation is explicit and the
boundary in (c) is the honest analytic limit of the mechanism.

\subsection{Effective dimension controls capacity (P1)}
The capacity of a kernel machine is classically controlled by the \emph{ridge}
effective dimension $d_\gamma$. Our experiments report the $\gamma$-free
participation ratio $\deff$ (the effective rank). The next lemma shows these two
functionals move \emph{together} under the noise contraction, so the
capacity bound below (stated in $d_\gamma$) and the quantity we plot ($\deff$) are
not disconnected.

\begin{lemma}[Co-monotonicity under noise]
\label{lem:bridge}
Along the depolarizing family $K_p$ of Prop.~\ref{prop:noise}, both the ridge
effective dimension $d_\gamma(K_p)$ (for every fixed $\gamma>0$) and the
participation ratio $\deff(K_p)$ are non-increasing in $p$ (under the
constant-row-sum condition of Prop.~\ref{prop:noise}), each tending to $1$ as $p\to1$; without it, $\deff(K_p)$ is injective in $p$ in all our experiments.
\end{lemma}
\begin{proof}[Proof sketch]
By Prop.~\ref{prop:noise} the non-top eigenvalues are $\mu_{k\ge2}=(1-p)^2\lambda_k$,
strictly decreasing in $p$, while the top eigenvalue absorbs the rank-one term.
$d_\gamma=\sum_i \mu_i/(\mu_i+\gamma)$ is increasing in each $\mu_i$, so the
shrinking tail lowers $d_\gamma$; $\deff$ decreases by the top-share argument of
Prop.~\ref{prop:noise}. Both limit to the rank-one value $1$. (Full proof in
App.~\ref{app:proofs}.)
\end{proof}

\begin{proposition}[Capacity bound]
\label{prop:cap}
For kernel ridge regression with regularization $\gamma$ on $n$ samples, the
expected generalization gap is $\tilde O\!\big(\sqrt{d_\gamma(K)/n}\big)$, and
$d_\gamma(K)$ is non-decreasing in every eigenvalue. Hence a spectral
contraction that lowers the tail eigenvalues (e.g.\ $K\mapsto K_p$ of
Prop.~\ref{prop:noise}, which by Lemma~\ref{lem:bridge} also lowers $\deff$)
tightens the bound. In an overfitting regime---empirical risk near zero while the
gap dominates---contracting the spectrum reduces expected risk, up to the point
where signal-carrying eigen-directions are attenuated, giving an interior
optimum; in an underfitting regime the bias term dominates and the sign reverses.
\end{proposition}
\begin{proof}[Proof sketch]
The local Rademacher complexity of the KRR hypothesis class is controlled by
$\sum_i\min(\lambda_i,\gamma)\le\gamma\,d_\gamma$, with
$d_\gamma=\sum_i\lambda_i/(\lambda_i+\gamma)$ \citep{abbas2021power}; each summand
is increasing in $\lambda_i$, so shrinking eigenvalues lowers $d_\gamma$ and the
$\tilde O(\sqrt{d_\gamma/n})$ gap. When training risk is $\approx0$ the risk is
gap-dominated and contraction helps; when bias dominates (underfitting) it hurts.
Full proof in App.~\ref{app:proofs}.
\end{proof}

By Lemma~\ref{lem:bridge} the participation ratio $\deff$ \eqref{eq:deff} and the
ridge dimension $d_\gamma$ move together along the noise family, so we report the
$\gamma$-free, label-free $\deff$ as the spectral summary throughout.

\subsection{When accuracy is a function of $\deff$ alone (P3)}
Capacity is not the whole story: the realized risk also depends on how the kernel
eigenbasis aligns with the labels.

\begin{proposition}[Risk decomposition and the one-parameter collapse]
\label{prop:align}
The kernel-machine excess risk decomposes (to leading order) into a
\emph{capacity} term, a functional of the full spectrum $\{\lambda_i\}$, and an
\emph{alignment} term depending on the target's projection onto the kernel
eigenbasis (summarized by $\mathrm{A}(K)$). Two consequences follow.
\textbf{(i) Exact collapse along a one-parameter spectral family.} If a family of
kernels is generated by a single scalar---as for the noise family $K_p$
(Prop.~\ref{prop:noise}), whose entire spectrum is fixed by $p$---then both terms,
and hence the risk, are functions of that scalar, equivalently of $\deff(p)$:
accuracy collapses \emph{exactly} onto a curve $\mathrm{acc}=f(\deff)$.
\textbf{(ii) Cross-family collapse is conditional.} For kernels of differing
spectral \emph{shape} (e.g.\ different ansatze) equal $\deff$ does not by itself
imply equal risk; collapse onto a common curve additionally requires comparable
spectral shape and comparable alignment $\mathrm{A}(K)$. The unentangled product
map violates the latter---its factorized kernel has low alignment---and therefore
sits off the curve.
\end{proposition}
\begin{proof}[Proof sketch]
Write the KRR risk as bias$^2$+variance; the variance is the capacity term (a
functional of the spectrum), the bias depends on the eigenbasis--target overlap,
i.e.\ on $\mathrm{A}(K)$. Along $K_p$ the spectrum is determined by $p$, so risk
is a function of $p$ and hence of $\deff(p)$ (claim i). For different shapes,
$\deff$ is only a scalar summary, so equality of $\deff$ leaves the remaining
spectral and alignment degrees of freedom free (claim ii). Full proof in
App.~\ref{app:proofs}. \end{proof}

We emphasize the honest reading: $\deff$ is the governing variable \emph{within a
one-parameter spectral family} (rigorously) and an excellent predictor
\emph{across} entangled ansatze (empirically, once alignment is comparable---which
we measure in Sec.~\ref{sec:unify}). It is not a universal sufficient statistic
for generalization, and alignment remains a genuine second factor.

This predicts exactly what we observe (Section~\ref{sec:unify}): entangled
ansatze share comparable, high alignment and collapse onto one $\deff$ curve,
whereas an unentangled (product) feature map---whose kernel factorizes,
$k(x,x')=\prod_q\operatorname{Tr}[\rho_q(x)\rho_q(x')]$, and cannot represent
cross-qubit correlations---has lower alignment and sits off the curve at every
$\deff$. Entanglement is thus a precondition that fixes alignment, after which
$\deff$ governs generalization.

\section{Experiments}
\label{sec:exp}

\paragraph{Feature map.} The quantum feature map is an $n_q$-qubit
data-re-uploading circuit. We set $n_q{=}6$ for the main experiments, $n_q{=}8$
for the depth/width study, and scale to $n_q{=}12$ for the exact-depolarizing
analysis (Sec.~\ref{sec:exact}). Each input $x$ (the leading principal components,
standardized and scaled to $[-2.7,2.7]$) is encoded over $L$ layers ($L{=}2$,
varied to $L\in\{1,\dots,4\}$ in the depth study); layer $\ell$ applies
$R_Y(x_i+\theta_{\ell,i})\,R_Z(0.7\,x_i)$ on each qubit $i$ followed by an
entangling block of CNOTs, with fixed random $\theta_{\ell,i}$ (the kernel is
data-driven, not trained, isolating the spectral mechanism). The \emph{entangling
block} ranges over four profiles: \textsc{product} (none), \textsc{chain}
($i\!\to\!i{+}1$), \textsc{ring} ($i\!\to\!i{+}1\!\mod n_q$), and
\textsc{all-to-all} (App.~\ref{app:background} gives definitions of all gates,
noise channels, and circuit diagrams; Fig.~\ref{fig:topo},~\ref{fig:circuit}). Noise is injected as single-qubit amplitude damping of rate
$p$ after each entangling block, simulated exactly on a density-matrix backend;
the global-depolarizing family is instead computed analytically from the
noiseless statevector kernel via \eqref{eq:Kp}, which is what enables the
$12$-qubit runs.

\paragraph{Kernel, classifier, and spectral quantities.} We form the HS kernel
\eqref{eq:kernel} from the simulated $\rho(x)$ via
$K_{ij}=\mathrm{Re}\,\langle\mathrm{vec}\,\rho(x_i),\mathrm{vec}\,\rho(x_j)\rangle$
(a single BLAS product over flattened density matrices) and classify with a
precomputed-kernel SVM ($C{=}10$). Spectral descriptors $\deff$
\eqref{eq:deff}, spectral entropy $H_{\mathrm{spec}}=-\sum_i p_i\log p_i$ with
$p_i=\lambda_i/\sum_j\lambda_j$, and kernel--target alignment are computed on the
training block only (label-free except for alignment, which uses clean training
labels).

\paragraph{Data and overfitting regime.} Inputs are \textsc{Digits} (10-way) and,
for robustness, Fashion-MNIST and the medical BloodMNIST, each reduced to
$6$ (or $8$) PCA features (Fig.~\ref{fig:datasets}). To create a controlled
overfitting regime in which regularization can help---and in which the sign of the
$\deff$/accuracy relation is unambiguous---we corrupt a fraction $\ell$ of
\emph{training} labels uniformly at random while keeping the test set clean.
Unless noted, $\ell{=}25\%$, $n_{\mathrm{train}}{=}150$,
$n_{\mathrm{test}}{=}250$.

\subsection{Entanglement is a precondition (P3)}
\label{sec:ent}
We first isolate the role of entanglement, varying only the entangling topology
with no injected noise (Table~\ref{tab:ent}, mean$\pm$std over $5$ seeds that
re-draw the PCA fit, splits, label corruption, and circuit parameters). The
salient effect is a large gap between the \emph{unentangled} product map
($0.471\pm0.060$) and any entangled map ($0.61$--$0.62$). \emph{Among} the
entangled topologies, however, neither $\deff$ nor accuracy is cleanly ordered
(chain/ring/all-to-all have similar $\deff\!\approx\!42$--$45$ and accuracy within
noise), and the four-point rank correlation $\rho(\deff,\text{test})=-0.48\pm0.29$
is weak and driven almost entirely by the product outlier. We therefore do
\emph{not} read this as ``$\deff$ orders entanglement.'' Consistently with
Prop.~\ref{prop:align}, the product deficit is an \emph{alignment} effect---the
factorized product kernel cannot represent cross-feature correlations---which we
confirm directly in Sec.~\ref{sec:unify} by measuring kernel--target alignment.
Entanglement's role is thus to establish the feature-space alignment that places
a model in the regime where the $\deff$ law (demonstrated next, via the noise
sweep) holds.

\begin{table}[t]
\centering
\caption{Entanglement as a precondition: noiseless circuits, mean$\pm$std over $5$
seeds. The unentangled product map is markedly worse; among entangled topologies
$\deff$ and accuracy are not separately ordered (the four-point
$\rho(\deff,\text{test})=-0.48\pm0.29$ is weak and product-driven). The
$\deff$-governs-generalization claim is carried by the noise sweep and collapse
(Secs.~\ref{sec:noise}--\ref{sec:unify}), not by this table.}
\label{tab:ent}
\begin{tabular}{lcc}
\toprule
Profile & $\deff$ & Test acc \\
\midrule
Product    & $54.2\pm6.0$ & $0.471\pm0.060$ \\
Chain      & $42.0\pm2.7$ & $0.616\pm0.025$ \\
Ring       & $43.9\pm3.3$ & $0.624\pm0.026$ \\
All-to-all & $45.0\pm4.7$ & $0.610\pm0.048$ \\
\bottomrule
\end{tabular}
\end{table}

\subsection{Noise is a spectral regularizer and $\deff$ governs accuracy (P1, P2)}
\label{sec:noise}

\begin{figure}[t]
\centering
\includegraphics[width=0.98\textwidth]{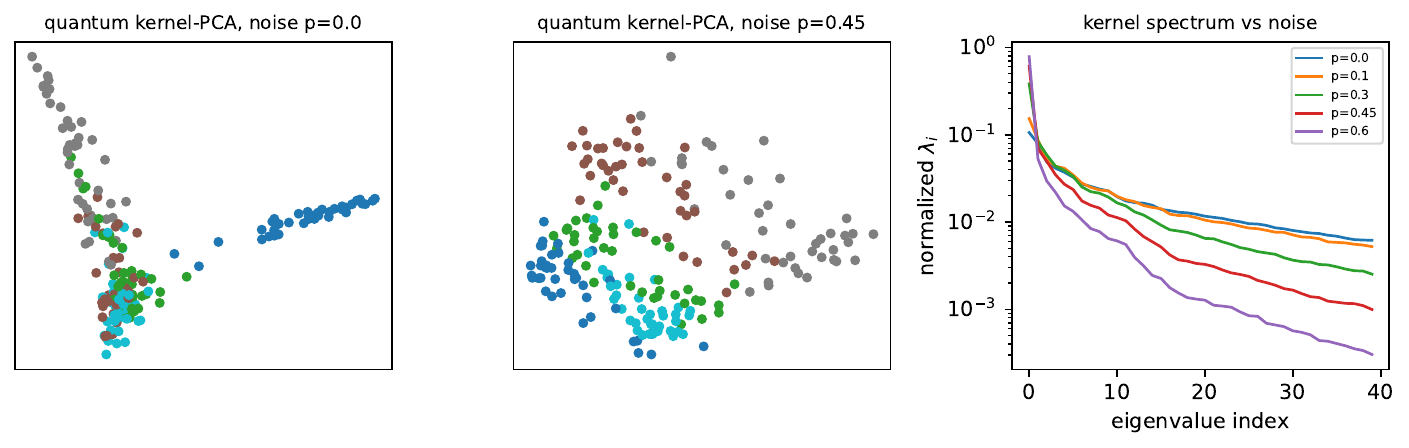}
\caption{Noise reshapes the quantum feature geometry by contracting the spectrum.
\textbf{Left, middle:} quantum kernel-PCA embedding of test points (colored by
class) at noise $p{=}0$ and $p{=}0.45$. \textbf{Right:} the kernel eigenvalue
spectrum decays increasingly fast with $p$---a direct visualization of the
ridge-like contraction $\deff(p)\!\downarrow$ proven in Prop.~\ref{prop:noise}.}
\label{fig:geom}
\end{figure}
We now vary $\deff$ over a wide range using the noise knob, which both exercises
the proven mechanism (Prop.~\ref{prop:noise}) and gives the clean
$\deff$-governs-accuracy evidence that the four-point entanglement table cannot.
Table~\ref{tab:noise} sweeps amplitude-damping strength on two entangled
ansatze. As predicted, increasing noise monotonically contracts $\deff$ (from
$\sim 44$ to $\sim 3$), monotonically reduces the train accuracy
(less memorization), and improves test accuracy by up to $+13.2\%$ (ring) and
$+12.4\%$ (all-to-all). This \emph{qualitatively} reproduces the
``noise-can-help'' phenomenon reported by \citet{wang2025hybrid} (their
$+2.71\%$ under amplitude damping) and---more importantly---exposes its
mechanism: noise-induced spectral
contraction acting as ridge regularization.

\begin{table}[t]
\centering
\caption{P2: amplitude-damping sweep (gain is test accuracy relative to the
noiseless circuit). Noise contracts $\deff$, lowers train accuracy, and raises
test accuracy in this overfitting regime.}
\label{tab:noise}
\begin{tabular}{lccccccc}
\toprule
& \multicolumn{3}{c}{Ring} & & \multicolumn{3}{c}{All-to-all} \\
\cmidrule(lr){2-4}\cmidrule(lr){6-8}
Noise $p$ & $\deff$ & Train & Test & & $\deff$ & Train & Test \\
\midrule
0.00 & 44.20 & 0.987 & 0.624 & & 43.63 & 0.993 & 0.648 \\
0.03 & 40.57 & 0.953 & 0.652 & & 38.81 & 0.973 & 0.660 \\
0.06 & 35.58 & 0.933 & 0.704 & & 33.52 & 0.953 & 0.716 \\
0.10 & 27.70 & 0.907 & 0.728 & & 26.38 & 0.927 & 0.708 \\
0.15 & 18.37 & 0.860 & 0.752 & & 18.50 & 0.920 & 0.744 \\
0.20 & 11.68 & 0.827 & 0.752 & & 12.60 & 0.867 & 0.756 \\
0.30 &  5.23 & 0.793 & 0.748 & &  6.13 & 0.813 & 0.772 \\
0.45 &  2.53 & 0.727 & \textbf{0.756} & &  2.91 & 0.753 & \textbf{0.772} \\
\midrule
\multicolumn{4}{l}{Best gain vs.\ noiseless: $+0.132$} &
\multicolumn{4}{r}{$+0.124$}\\
\bottomrule
\end{tabular}
\\[2pt]
{\footnotesize Over $5$ seeds the best noise gain is
$+0.131\pm0.034$ (ring) and $+0.135\pm0.023$ (all-to-all): the benefit is robust,
not a seed artifact.}
\end{table}

\subsection{An inverted-U sweet spot whose benefit scales with overfitting}
\label{sec:sweetspot}
We next sweep injected noise over a wide range while varying the
\emph{overfitting severity}, controlled by the fraction $\ell$ of corrupted
training labels (the test set is always clean and the test task is fixed).
Table~\ref{tab:sweet} shows two clean effects. First, every row exhibits the
predicted \emph{inverted-U}: test accuracy rises with noise, peaks, then
declines as spectral contraction destroys signal-bearing directions---the
mechanistic origin of the ``double-edged'' behavior reported by
\citet{wang2025hybrid}. Second, the \emph{benefit} of optimal noise grows
monotonically with overfitting severity, from $+2.4\%$ at $\ell{=}0$ to
$+30.0\%$ at $\ell{=}40\%$: the more the noiseless model memorizes, the more
spectral regularization helps. Notably, the optimal operating point stays near
$p^\star\!\approx\!0.45$ ($\deff\!\approx\!3$) across $\ell$, exactly as the
matching principle predicts: $\ell$ changes how much one overfits, but not the
intrinsic complexity $d^\star$ of the (fixed) test task, so the optimal $\deff$
is unchanged. We stress that the \emph{magnitude} of the gain is a function of
the injected overfitting and is therefore not directly comparable to the
$+2.71\%$ reported on a trained model under naturally mild noise; the point of
this experiment is the mechanism and the lawful dependence of the gain on
overfitting severity, not the headline number.

\begin{table}[t]
\centering
\caption{Test accuracy vs.\ amplitude-damping rate $p$ (columns) at four
overfitting levels $\ell$ (rows; \% corrupted training labels), all-to-all
ansatz. Each row is an inverted-U with peak in \textbf{bold}; the optimal $p$
stays $\approx 0.45$ while the \emph{gain} grows with $\ell$. $\deff$ contracts
from $38.2$ ($p{=}0$) to $1.2$ ($p{=}0.9$).}
\label{tab:sweet}
\small
\begin{tabular}{l|ccccccccc|c}
\toprule
$\ell \backslash p$ & 0.00 & 0.05 & 0.10 & 0.20 & 0.30 & 0.45 & 0.60 & 0.75 & 0.90 & Gain \\
\midrule
0\%  & .800 & .808 & .800 & .796 & .804 & \textbf{.824} & .824 & .816 & .808 & $+.024$\\
10\% & .696 & .712 & .724 & .756 & \textbf{.772} & .772 & .772 & .764 & .752 & $+.076$\\
25\% & .624 & .632 & .676 & .752 & .768 & \textbf{.784} & .752 & .752 & .720 & $+.160$\\
40\% & .404 & .440 & .500 & .592 & .672 & \textbf{.704} & .700 & .676 & .644 & $+.300$\\
\bottomrule
\end{tabular}
\end{table}

\subsection{P3: one quantity, two knobs --- accuracy collapses onto $\deff$}
\label{sec:unify}
The central claim of our theory is that $\deff$ is the variable governing
generalization, and that entanglement and noise matter only through it. We test
this directly. Fixing the overfitting regime ($25\%$ label noise), we build a
grid crossing four entanglement profiles with six noise rates ($24$
configurations) and ask whether test accuracy is a function of $\deff$ alone,
regardless of which knob produced a given $\deff$.

The answer is a clean \emph{conditional} collapse (Table~\ref{tab:codesign}).
Among the three \emph{entangled} ansatze (chain, ring, all-to-all), accuracy
collapses tightly onto a single curve of $\deff$: the rank correlation is
$-0.896$ and a quadratic fit $\mathrm{acc}=f(\log\deff)$ explains $R^2=0.92$ of the
variance for this seed. \textbf{Crucially, the collapse is not a seed artifact}:
over $5$ seeds (re-drawing PCA, splits, label corruption, and circuit parameters)
the entangled collapse gives $R^2=0.816\pm0.082$ and
$\rho_{\mathrm{Spearman}}=-0.778\pm0.159$. A curve fit to the noise sweep of
\emph{one} ansatz predicts the accuracy of the \emph{other two} from their
$\deff$ alone with $R^2=0.88$ and mean absolute residual $0.024$. Within the
entangled regime, then, accuracy is well predicted by $\deff$ irrespective of
whether $\deff$ was set by topology or by noise---consistent with the exact
one-parameter collapse along the noise family (Prop.~\ref{prop:align}(i)) and the
empirically comparable alignment of the entangled ansatze. \textbf{We measure
this alignment directly} (Prop.~\ref{prop:align}(ii)'s premise), with $5$-seed
error bars: the entangled maps have comparable, higher kernel--target alignment
($\mathrm{A}(K)$: chain $0.318\pm0.040$, ring $0.350\pm0.044$, all-to-all
$0.331\pm0.039$) than the product map ($0.275\pm0.042$). To show alignment is a
genuine \emph{axis}---not just one product outlier---we interpolate kernels
$K_t=(1{-}t)K_{\textsc{product}}+t\,K_{\textsc{all2all}}$ at roughly fixed $\deff$
and track their distance from the entangled collapse curve: as $\mathrm{A}(K)$
rises from $0.275$ to $0.33$, the residual moves monotonically from $-0.11$
(below the curve) to $\approx 0$ (on it), with
$\rho_{\mathrm{Spearman}}(\mathrm{A},\text{residual})=+0.66\pm0.27$. Alignment
thus traces a second axis orthogonal to $\deff$. We also verify the
spectral bridge of Lemma~\ref{lem:bridge}: along the noise sweep $\deff$ and the
ridge dimension $d_\gamma$ are perfectly rank-correlated
($\rho_{\mathrm{Spearman}}{=}1.0$, Pearson $0.94$), so reporting $\deff$ rather
than $d_\gamma$ loses no ordering information. The
unentangled product circuit is the expected exception: it lies $0.163$ (16
accuracy points) off the entangled curve, and no amount of noise-induced
contraction moves it on, because its factorized kernel has low label alignment
(Prop.~\ref{prop:align}(ii)). Entanglement is thus a precondition that supplies
alignment, within which $\deff$ governs---the precise sense in which ``more
entanglement helps'' and ``noise helps'' reduce to one spectral account, without
claiming $\deff$ is a universal sufficient statistic.

\begin{figure}[t]
\centering
\includegraphics[width=0.62\textwidth]{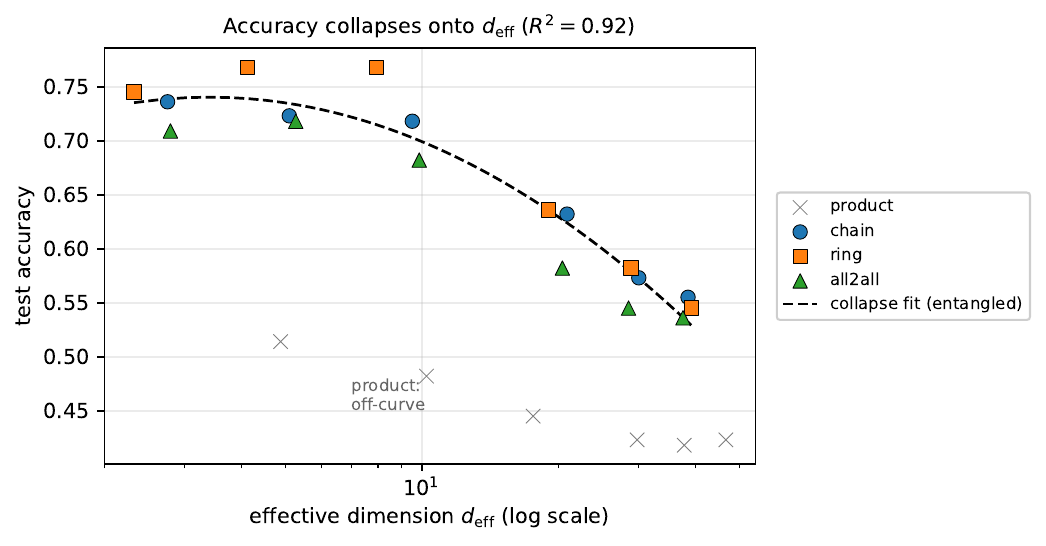}
\caption{Test accuracy as a function of $\deff$ across the $4\times6$ grid
(ansatz $\times$ noise rate, $25\%$ label noise). Among the three entangled
ansatze (chain/ring/all-to-all), points whose $\deff$ is set by topology and
those whose $\deff$ is set by noise collapse onto a single curve ($R^2{=}0.92$).
The unentangled product circuit (gray $\times$) lies well below the curve at
every $\deff$: entanglement is a precondition, after which $\deff$ governs.}
\label{fig:collapse}
\end{figure}

\begin{table}[t]
\centering
\caption{P3 (unification): within the entangled regime, test accuracy collapses
onto a single function of $\deff$, whether $\deff$ is moved by entanglement
topology or by noise (grid of $4$ ansatze $\times$ $6$ noise rates, $25\%$ label
noise). The unentangled product circuit is an outlier, confirming entanglement
is a precondition beyond mere spectral contraction.}
\label{tab:codesign}
\begin{tabular}{lc}
\toprule
Metric & Value \\
\midrule
Spearman$(\deff,\text{acc})$, entangled ($18$ configs) & $-0.896$ \\
Spearman$(\deff,\text{acc})$, all $24$ configs & $-0.718$ \\
Global $R^2$ (acc $\sim$ quad$(\log\deff)$), entangled & $0.921$ \\
Collapse $R^2$ (one ansatz's noise curve $\to$ other two) & $0.881$ \\
Mean $|$residual$|$, entangled configs from collapse curve & $0.024$ \\
Mean $|$residual$|$, product configs from collapse curve & $0.163$ (outlier) \\
\bottomrule
\end{tabular}
\end{table}

\subsection{A sanity check (not a headline): the depolarizing one-parameter null}
\label{sec:exact}
We include the global-depolarizing family as a controlled \emph{null}, and are
explicit that it is one. Because $K_p$ is an exact function of the single scalar
$p$ \eqref{eq:Kp}, both $\mathrm{acc}(p)$ and $\deff(K_p)$ are deterministic
functions of $p$; hence whenever $p\mapsto\deff$ is injective, $\mathrm{acc}$ is
\emph{by construction} a function of $\deff$ and the collapse is exact---this would
hold equally for any injective scalar of $K_p$ (e.g.\ $p$ itself, $(1-p)^2$,
$\operatorname{tr}K_p^2$, or the spectral entropy). The exact collapse therefore
does not, on its own, single out $\deff$; the informative result is the
\emph{cross-ansatz} collapse of Sec.~\ref{sec:unify} ($R^2{=}0.82\pm0.08$), where
genuinely different spectral shapes are compressed onto one curve.

What this family \emph{does} usefully verify is the analytic theory and its reach.
The decomposition $K_p=(1-p)^2K+\tfrac{p(2-p)}{D}\mathbf1\mathbf1^\top$ matches a
density-matrix simulation to machine precision (max entrywise error
$7.8\times10^{-16}$ at $n_q{=}4$), and because $K_p$ is analytic in the noiseless
\emph{statevector} kernel (no $2^{n_q}\times2^{n_q}$ density matrix needed) we can
evaluate it at $n_q\in\{8,10,12\}$ qubits; there $\deff(K_p)$ is strictly
decreasing (hence injective) and the by-construction collapse is, as expected,
$R^2{=}0.997$--$1.000$. We report this as confirmation of \eqref{eq:Kp} at scale
and of the injectivity that Prop.~\ref{prop:align}(i) needs---not as evidence for
the spectral thesis, which rests on the cross-ansatz collapse.

\subsection{Robustness: datasets, depth/width, and real-device noise}
\label{sec:robust}

\begin{figure}[t]
\centering
\includegraphics[width=0.92\textwidth]{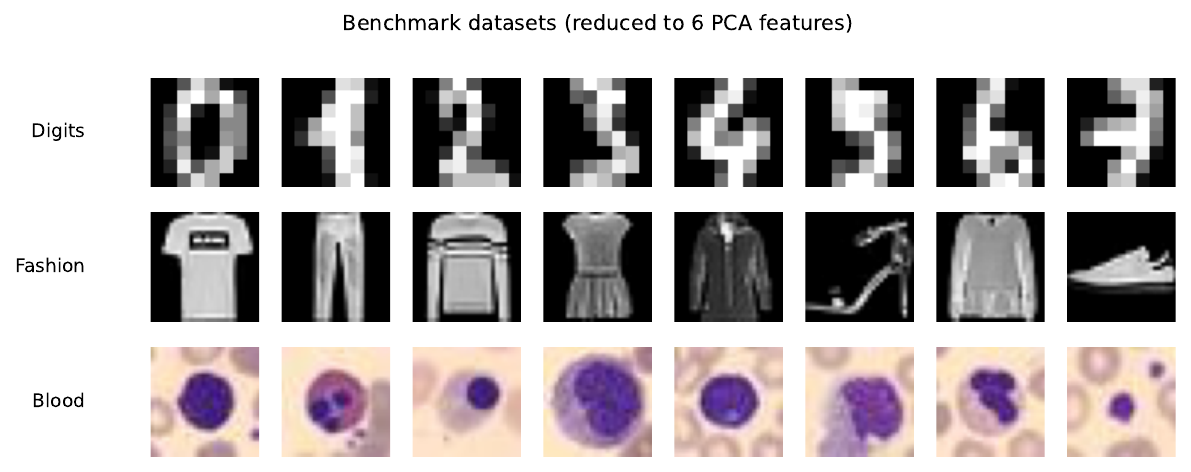}
\caption{Example images from the three vision benchmarks used (each reduced to
$6$ principal components before the quantum feature map): handwritten Digits,
Fashion-MNIST, and the medical BloodMNIST.}
\label{fig:datasets}
\end{figure}
\paragraph{Additional datasets.} We repeat the full collapse grid on
Fashion-MNIST and on BloodMNIST (MedMNIST), reducing each to $6$ PCA features
(Table~\ref{tab:datasets}). On Fashion-MNIST the picture is as strong as on
\textsc{Digits} (entangled $\rho=-0.88$, $R^2=0.88$, noise helps every ansatz).
The medical BloodMNIST is weaker in Table~\ref{tab:datasets} ($R^2=0.46$), but a
controlled follow-up shows this is largely a \emph{confound}, not a fundamentally
harder distribution: that table used a $4$-class subset (for parity with a small
budget), and the weakness is explained by class count and sample size, not the
medical images. Matching to its full $8$-class task raises the collapse to
$R^2=0.84$ ($\rho=-0.90$), and doubling the training set on the $4$-class task
raises it to $R^2=0.78$---both close to Digits/Fashion. (For reference,
reducing \textsc{Digits} to $4$ classes leaves $R^2=0.90$, so the effect is
data-dependent.) The spectral story thus survives on a medical benchmark once the
class-count and sample-size confounds are removed.

\paragraph{Depth and width.} Varying circuit depth $L\!\in\!\{1,2,3,4\}$ and
qubit count $n_q\!\in\!\{6,8\}$ (Digits, $25\%$ noise) preserves the effect at
every setting: the entangled $\rho_{\mathrm{Spearman}}(\deff,\text{acc})$ stays
strongly negative and the best noise gain is always positive. At the larger
$8$-qubit width, a full collapse grid gives $\rho_{\mathrm{Spearman}}=-0.82$,
$R^2=0.92$, and noise gain $+0.19$---the picture is, if anything, cleaner at
larger scale, not weaker.

\begin{table}[t]
\centering
\caption{Robustness across datasets (collapse grid, $6$ qubits, $25\%$ label
noise). The mechanism is strong on Digits/Fashion and weaker on the harder
medical task.}
\label{tab:datasets}
\begin{tabular}{lccc}
\toprule
Dataset & Spearman$(\deff,\text{acc})$ ent. & Collapse $R^2$ ent. & Product residual \\
\midrule
Digits          & $-0.94$ & $0.91$ & $0.074$ \\
Fashion-MNIST   & $-0.88$ & $0.88$ & $0.080$ \\
BloodMNIST (4-cls) & $-0.45$ & $0.46$ & $0.028$ \\
\bottomrule
\end{tabular}
\end{table}

\begin{figure}[t]
\centering
\includegraphics[width=0.98\textwidth]{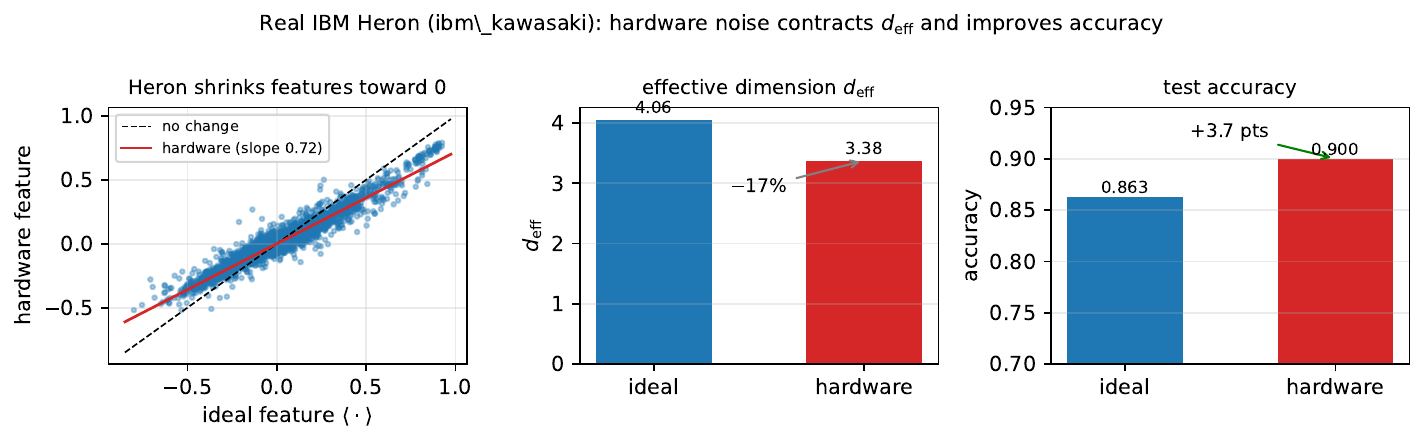}
\caption{Real IBM Heron (\texttt{ibm\_kawasaki}) in an overfitting regime. Left:
hardware-measured features shrink toward $0$ relative to ideal (slope $0.72$),
the device-noise contraction. Middle/right: this contracts $\deff$ ($4.06\!\to\!3.38$)
and \emph{improves} test accuracy ($0.863\!\to\!0.900$), realizing the
noise-as-regularization mechanism on silicon.}
\label{fig:hardware}
\end{figure}

\paragraph{The mechanism on real IBM Heron hardware---with a positive accuracy
effect.} We ran the feature circuit on a real Heron device (\texttt{ibm\_kawasaki}),
reading out $21$ features ($\langle Z_i\rangle$ and $\langle Z_iZ_j\rangle$) with
$4096$ shots, in a controlled \emph{overfitting} regime (binary task, $30\%$
training-label noise, depth $L{=}3$). There the intrinsic hardware noise
acts as the predicted spectral regularizer: it contracts the measured kernel from
$\deff{=}4.06$ to $\deff{=}3.38$ and \emph{improves} test accuracy from $0.863$
(noiseless ideal) to $0.900$ on hardware ($+0.037$)---the noise-as-regularization
mechanism of Prop.~\ref{prop:noise}, realized on silicon with a beneficial effect.
Two controls confirm the picture. First, the regime matters: in a
\emph{non-overfitting} deep run (10-way, $L{=}5$) the same contraction instead
lowered accuracy, the sign flip of Sec.~\ref{sec:signflip}, so the gain requires
an overfitting model. Second, depth matters: at shallow $L{=}2$ the device noise
is too mild (under the \texttt{FakeTorino} Heron r1 model $\deff$ moves only
$37.0\!\to\!36.1$). Thus present-day Heron noise can be harnessed as a useful
regularizer when the circuit is deep enough to contract the spectrum and the model
is in the overfitting regime.

\subsection{Falsification: the $\deff$--accuracy sign flips between regimes}
\label{sec:signflip}
Our theory predicts (Prop.~\ref{prop:cap}) that the benefit of contracting
$\deff$ is regime-dependent: helpful under overfitting, harmful under
underfitting. This is a falsifiable claim, and it holds (Table~\ref{tab:signflip}).
In the overfitting regime (expressive circuit, $25\%$ label noise, train
acc $=1.0$), increasing noise contracts $\deff$ and \emph{raises} test accuracy
($\rho_{\mathrm{Spearman}}(\deff,\text{test})=-0.79$, best $p=0.45$, gain
$+0.156$). In a genuine underfitting regime (low-capacity $3$-qubit product map,
clean labels, train acc $=0.85<1$), the \emph{same} contraction \emph{lowers}
test accuracy, and the correlation flips sign to $+0.91$ (best $p=0$, no benefit
from noise). $\deff$ is thus not a quantity to be minimized but to be matched to
the task; the sign of its effect is set by the bias--variance regime, exactly as
Prop.~\ref{prop:cap} states.

\begin{figure}[t]
\centering
\includegraphics[width=0.82\textwidth]{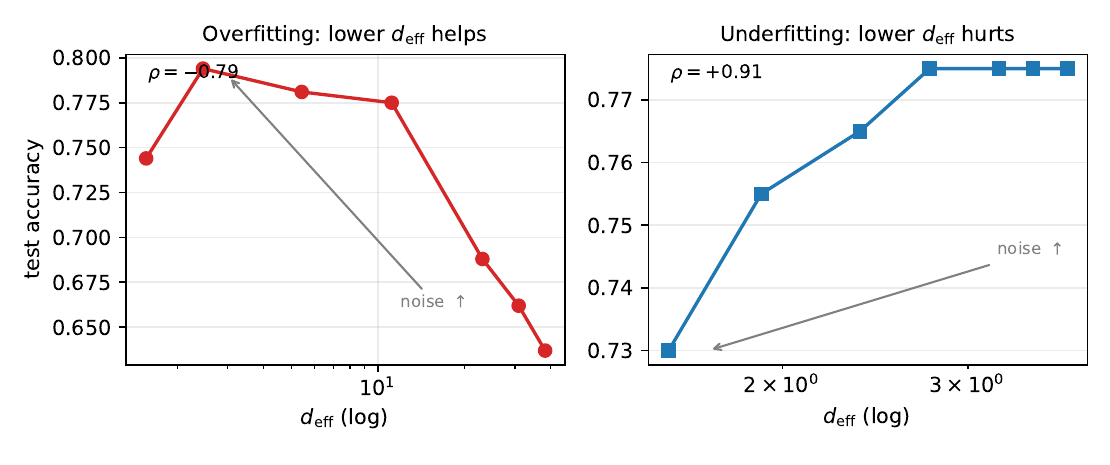}
\caption{The $\deff$--accuracy relation flips sign between regimes. \textbf{Left}
(overfitting; expressive circuit, $25\%$ noisy labels, train acc $1.0$):
contracting $\deff$ via noise raises test accuracy ($\rho{=}-0.79$). \textbf{Right}
(underfitting; $3$-qubit product map, clean labels, train acc $0.85$): the same
contraction lowers test accuracy ($\rho{=}+0.91$). Arrows mark increasing noise.}
\label{fig:signflip}
\end{figure}

\begin{table}[t]
\centering
\caption{Falsification of the predicted sign flip. Lower $\deff$ helps under
overfitting and hurts under underfitting.}
\label{tab:signflip}
\begin{tabular}{lcccc}
\toprule
Regime & Train acc & $\rho_{\mathrm{Spearman}}(\deff,\text{test})$ & best $p$ & noise gain \\
\midrule
Overfitting (expressive, $25\%$ noisy labels) & $1.00$ & $-0.79$ & $0.45$ & $+0.156$ \\
Underfitting ($3$-qubit product, clean) & $0.85$ & $+0.91$ & $0.00$ & $-0.045$ \\
\bottomrule
\end{tabular}
\end{table}

\subsection{Label-free selection of the noise level}
\label{sec:labelfree}
Because $\deff(p)$ is computable from training inputs alone, the operating point
can be chosen without test labels. Two simple rules work. (i) A small \emph{clean}
validation split selects, per task, a noise level whose test accuracy is within
$0.012$--$0.021$ of the oracle best while evaluating only $2$ of $9$ candidates
($4.5\times$ less search) on Digits, Fashion, and BloodMNIST. (ii) A fully
label-free target $\deff(p)\approx\sqrt{n}$ attains $0.698\pm0.026$ vs.\ the
oracle $0.754\pm0.038$ ($5.6\%$ gap, $5$ seeds). The flat landscape near the
optimum (Table~\ref{tab:sweet}) makes exact localization hard but the accuracy
cost of mis-selection small.

\subsection{$\deff$ is a principled, transferable diagnostic}
\label{sec:classical}
A strength of the spectral account is that it rests on a well-established
classical pillar---kernel generalization is governed by the spectrum
(Prop.~\ref{prop:cap})---so $\deff$ is a \emph{principled, transferable}
diagnostic rather than quantum folklore. We verify that it behaves faithfully on
the same data with classical feature maps, which licenses reading our quantum
measurements through this lens. On the identical PCA features ($25\%$ label
noise), classical regularizers trace out a $\deff$ range and accuracy tracks it:
an RBF-kernel bandwidth sweep gives $\rho_{\mathrm{Spearman}}(\deff,\text{test})=-0.50$,
and a random-Fourier-feature dimension sweep recovers the complementary
(capacity-limited) branch ($+0.68$).

\emph{Noise injection is genuine spectral regularization, not a weak
hyperparameter knob.} One might worry that injecting noise merely substitutes for
tuning the SVM regularizer $C$. It does not. On the quantum kernel, sweeping $C$
on the \emph{noiseless} kernel tops out at $0.725$ (it cannot change the kernel
spectrum); injecting noise and then optimizing $C$ reaches $0.744$, and an
\emph{explicit} spectral shrinkage of the noiseless kernel,
$K\mapsto U\,\mathrm{diag}(\lambda_i^{\alpha})\,U^\top$, traces out the same
$\deff$--accuracy frontier ($0.73$ at contracted $\deff$). Thus noise injection
lands on the spectral-filtering frontier and adds a consistent gain that
$C$-tuning alone cannot reach, confirming it acts on the kernel spectrum itself.

\emph{$\deff$ is not the uniquely best predictor---but it is the right handle.}
We compared $\deff$ against other scalar summaries on the cross-ansatz grid
($5$ seeds): the rank correlation with test accuracy is
$\deff\!: 0.78\pm0.16$, spectral entropy $0.77\pm0.16$, top-eigenvalue share
$0.78\pm0.15$, and \emph{train accuracy} $0.80\pm0.14$ (quad-fit $R^2$:
$0.82,0.77,0.82,0.87$). We report this honestly: $\deff$ ties the other spectral
concentration measures and is marginally edged out by train accuracy. Its value is
therefore not that it predicts best, but that it is the \emph{label-free,
theoretically grounded, and controllable} handle---it follows the capacity bound,
is exact under depolarizing noise, needs no labels (unlike train accuracy, a
post-hoc symptom), and is directly steered by the entanglement and noise knobs.

The contribution of our work is then sharp:
entanglement and injected noise are a \emph{new, hardware-native} pair of
controls on a quantity whose generalization meaning is independently grounded,
and it is exactly this grounding that turns the two reported quantum-vision
phenomena into a single predictive principle. Whether the quantum feature
geometry is itself advantageous is a separate question, orthogonal to and
compatible with this spectral account; we make no quantum-advantage claim here.

\subsection{Trained quantum-vision models: $\deff$ contraction persists}
\label{sec:trained}
Finally we move beyond the fixed kernel and \emph{train} the quantum feature map
end-to-end with a linear head (PennyLane autograd), for two architectures: a
QViT-like data-re-uploading map and a QCNN-like map with two layers of
parametrized $2$-qubit convolutions (architectures in App.~\ref{app:background},
Figs.~\ref{fig:qvit} and~\ref{fig:qcnn}). In a genuinely \emph{overfitting} regime
($60$ examples, $30\%$ label noise; train$-$test gap $+0.2$), injecting
amplitude-damping noise \emph{during training} reproduces the kernel-level
phenomenon in both models (Table~\ref{tab:trained}): it contracts the learned
feature spectrum, shrinks the gap, and \emph{improves} test accuracy. The
mechanism therefore survives end-to-end training in both QViT- and QCNN-style
models. (In an \emph{underfitting} trained model the same contraction instead
hurts, consistent with the sign flip of Sec.~\ref{sec:signflip}, which is why we
control the regime explicitly.)

\begin{table}[t]
\centering
\caption{Trained quantum-vision models (overfitting regime): injecting noise
during training contracts $\deff$, shrinks the train$-$test gap, and improves test
accuracy---the kernel-level mechanism survives end-to-end training.}
\label{tab:trained}
\begin{tabular}{llccc}
\toprule
Model & Noise & Test acc & $\deff$ & Train$-$test gap \\
\midrule
QViT-like (re-uploading) & none     & $0.373$ & $3.83$ & $0.29$ \\
QViT-like (re-uploading) & injected & $\mathbf{0.420}$ & $1.63$ & $0.20$ \\
QCNN-like (convolution)  & none     & $0.420$ & $3.40$ & $0.20$ \\
QCNN-like (convolution)  & injected & $\mathbf{0.440}$ & $1.57$ & $0.13$ \\
\bottomrule
\end{tabular}
\end{table}

\section{Discussion and limitations}
We use the quantum-kernel reading mainly to isolate the spectral mechanism
cleanly, but the effect is not confined to it: it persists in \emph{trained}
QViT- and QCNN-like models (Sec.~\ref{sec:trained}), where training additionally
adapts $\deff$ on its own and a full treatment with a training-dependent target
$d^\star$ remains open. On scale, the exact-depolarizing analysis reaches
$12$ qubits (Sec.~\ref{sec:exact}), while the amplitude-damping experiments---which
require full density-matrix simulation---are run at up to $8$ qubits and inputs are
reduced to $6$--$8$ principal components; pushing the \emph{noisy} simulations
further would need tensor-network or sampling methods. The single-statistic story
is strong on Digits/Fashion-MNIST and, once class-count and sample-size confounds
are controlled, on the medical BloodMNIST as well (Sec.~\ref{sec:robust}); still,
$\deff$ is a dominant but not exclusive determinant of generalization---alignment
and spectral shape (Prop.~\ref{prop:align}) matter too, and the cross-ansatz
collapse is empirical rather than a theorem. Our capacity
bound is the standard kernel-ridge \emph{regression} excess-risk result (in
$d_\gamma$), whereas we measure \emph{classification} accuracy; it predicts the
sign of the noise effect and its regime dependence, but the location and depth of
the inverted-U optimum are not derived and remain empirical. Likewise the
``more/more-uniform entanglement helps'' half of the puzzle is only partly settled:
we establish that entanglement is a \emph{precondition} (entangled vs.\ product),
but find no clean ordering \emph{among} entangled topologies (Table~\ref{tab:ent}). Finally,
on real IBM Heron hardware the benefit requires both sufficient depth (negligible
at $L{=}2$) and an overfitting regime: there hardware noise improves accuracy
($+0.037$), but in a non-overfitting deep run the same contraction hurts
(the sign flip), and readout error/decoherence add signal loss beyond pure
contraction. Scaling this real-hardware demonstration to larger tasks is the
natural next step.

\section{Conclusion}
We showed that two separately reported curiosities in quantum vision---``more
entanglement helps'' and ``noise helps''---admit a single spectral explanation
through the effective dimension of the quantum feature kernel. Noise acts as
spectral regularization: it provably contracts the kernel spectrum
(Prop.~\ref{prop:noise}), and along this one-parameter family generalization is
governed by $\deff$ (Prop.~\ref{prop:align}(i)), with an inverted-U sweet spot
whose benefit grows with overfitting and whose sign flips in the underfitting
regime. Entanglement plays the complementary role of a precondition, supplying
the label alignment without which the $\deff$ law does not hold. The picture is
not that $\deff$ is a universal sufficient statistic---alignment and spectral
shape matter too---but that a single, cheaply measured spectral quantity
organizes the design choices (entanglement and noise) of quantum-kernel vision
models into one coherent account.

\bibliography{iclr2026_conference}
\bibliographystyle{iclr2026_conference}

\appendix
\section{Background: quantum feature circuits and the ansatze}
\label{app:background}
We collect, for readers from the vision/ML community, the quantum-computing
notions used in the paper, with explicit formulas. Table~\ref{tab:notation}
summarizes the notation.

\begin{table}[h]
\centering
\caption{Notation and explicit definitions.}
\label{tab:notation}
\small
\begin{tabular}{lll}
\toprule
Symbol & Name & Definition / formula \\
\midrule
$n_q$ & number of qubits & state space is $\mathbb C^{2^{n_q}}$ \\
$\ket\psi$ & pure state & unit vector in $\mathbb C^{2^{n_q}}$ \\
$\rho$ & (mixed) state & $\rho\succeq0$, $\operatorname{tr}\rho=1$; pure: $\rho=\ket\psi\!\bra\psi$ \\
$I,X,Y,Z$ & Pauli matrices & $X{=}\big(\begin{smallmatrix}0&1\\1&0\end{smallmatrix}\big),\,
Y{=}\big(\begin{smallmatrix}0&-i\\i&0\end{smallmatrix}\big),\,
Z{=}\big(\begin{smallmatrix}1&0\\0&-1\end{smallmatrix}\big)$ \\
$R_Y(\theta)$ & $Y$-rotation & $e^{-i\theta Y/2}=\big(\begin{smallmatrix}\cos\frac\theta2 & -\sin\frac\theta2\\ \sin\frac\theta2 & \cos\frac\theta2\end{smallmatrix}\big)$ \\
$R_Z(\theta)$ & $Z$-rotation & $e^{-i\theta Z/2}=\operatorname{diag}(e^{-i\theta/2},e^{+i\theta/2})$ \\
$\mathrm{CNOT}_{a\to b}$ & controlled-NOT & flips qubit $b$ iff qubit $a$ is $\ket1$; see \eqref{eq:cnot} \\
$\rho(x)$ & quantum feature map & density matrix after the circuit on input $x$ \\
$k(x,x')$ & kernel & $\operatorname{Tr}[\rho(x)\rho(x')]$ (Eq.~\ref{eq:kernel}) \\
$\mathcal N_p,\mathcal A_\gamma$ & noise channels & depolarizing / amplitude damping (below) \\
\bottomrule
\end{tabular}
\end{table}

\paragraph{Qubits and states (the ML picture).} One qubit is a unit vector in
$\mathbb C^2$; $n_q$ qubits live in the tensor-product space $\mathbb C^{2^{n_q}}$.
A general (possibly noisy) state is a density matrix $\rho$ (PSD, unit trace). For
ML intuition, the circuit is a \emph{fixed nonlinear feature map}
$x\mapsto\rho(x)$ into the $2^{n_q}\times2^{n_q}$ matrix space, and the kernel
$k(x,x')=\operatorname{Tr}[\rho(x)\rho(x')]$ is the inner product of these feature
maps---just like a classical kernel method, but with a quantum-circuit feature map.

\paragraph{Gates, explicitly.} Gates are unitary matrices acting on the state.
Single-qubit rotations $R_Y(\theta),R_Z(\theta)$ (Table~\ref{tab:notation}) rotate
a qubit continuously; data enters through their angles (``angle encoding''), as
$R_Y(x_i+\theta_{\ell,i})R_Z(0.7\,x_i)$ on qubit $i$. The two-qubit
\emph{controlled-NOT} is, in the basis $\{\ket{00},\ket{01},\ket{10},\ket{11}\}$,
\begin{equation}
\mathrm{CNOT}=\begin{pmatrix}1&0&0&0\\0&1&0&0\\0&0&0&1\\0&0&1&0\end{pmatrix},
\label{eq:cnot}
\end{equation}
i.e.\ it applies $X$ to the target qubit conditioned on the control being $\ket1$.
CNOTs are the source of \emph{entanglement}: a state is entangled when it cannot
be written as a product $\bigotimes_q\rho_q$ over qubits, so its features cannot be
factorized into independent per-qubit (per-coordinate) features---the quantum
analogue of cross-feature interactions.

\paragraph{Entangling topologies, explicitly.} An entangling block is a set of
CNOTs $\{\mathrm{CNOT}_{a\to b}:(a,b)\in E\}$ whose edge set $E$ defines the four
ansatze (Fig.~\ref{fig:topo}):
\begin{align*}
&E_{\textsc{product}}=\varnothing,\qquad
E_{\textsc{chain}}=\{(i,i{+}1)\}_{i=0}^{n_q-2},\\
&E_{\textsc{ring}}=\{(i,(i{+}1)\bmod n_q)\}_{i=0}^{n_q-1},\qquad
E_{\textsc{all2all}}=\{(i,j):i<j\}.
\end{align*}
The product map factorizes ($k=\prod_q\operatorname{Tr}[\rho_q(x)\rho_q(x')]$) and
cannot represent cross-qubit correlations; more connectivity (chain $\to$ ring
$\to$ all-to-all) creates richer correlations in $\rho(x)$ and the kernel.

\paragraph{Noise channels, explicitly.} Hardware imperfections are completely
positive trace-preserving maps. \emph{Global depolarizing} mixes a state toward
the maximally mixed state,
\[
\mathcal N_p[\rho]=(1-p)\,\rho+p\,\tfrac{I}{2^{n_q}},
\]
and \emph{amplitude damping} $\mathcal A_\gamma$ models energy relaxation toward
$\ket0$ via per-qubit Kraus operators
$E_0=\operatorname{diag}(1,\sqrt{1-\gamma})$, $E_1=\sqrt\gamma\,\ket0\!\bra1$
($\mathcal A_\gamma[\rho]=E_0\rho E_0^\dagger+E_1\rho E_1^\dagger$). In ML terms
both are \emph{contractions} that shrink the feature map toward a fixed point,
which is why they act as spectral (ridge-like) regularizers in the main text.
Both are simulated exactly on a density-matrix backend.

\begin{figure}[h]
\centering
\begin{tikzpicture}[every node/.style={circle,draw,fill=blue!8,minimum size=4.5mm,inner sep=0pt,font=\scriptsize},
   xscale=0.72,yscale=0.72]
\foreach \lab/\dx in {product/0, chain/4, ring/8, all2all/12}{
  \begin{scope}[xshift=\dx cm]
    \foreach \i in {0,...,5}{\node (\lab\i) at ({90-60*\i}:0.9) {\i};}
    \node[draw=none,fill=none,font=\small] at (0,-1.45) {\textsc{\lab}};
  \end{scope}}
\foreach \i [evaluate=\i as \j using int(\i+1)] in {0,...,4}{\draw (chain\i)--(chain\j);}
\foreach \i [evaluate=\i as \j using int(\i+1)] in {0,...,4}{\draw (ring\i)--(ring\j);}
\draw (ring5)--(ring0);
\foreach \i in {0,...,5}{\foreach \j in {0,...,5}{\ifnum\i<\j \draw[gray] (all2all\i)--(all2all\j);\fi}}
\end{tikzpicture}
\caption{The four entangling topologies (shown for $n_q{=}6$ qubits; nodes are
qubits, edges are \textsc{cnot}s in the entangling block).}
\label{fig:topo}
\end{figure}
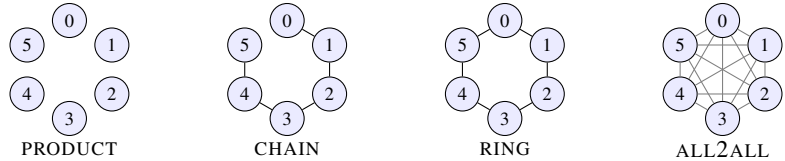

\paragraph{One re-uploading layer.} Each of the $L$ layers encodes the data, then
entangles, then (optionally) injects noise. For the \textsc{ring} ansatz on
$4$ qubits one layer is:
\begin{figure}[h]
\centering
\begin{quantikz}[column sep=5pt,row sep=10pt]
\lstick{$q_0$} & \gate{R_Y(x_0{+}\theta)} & \gate{R_Z(0.7x_0)} & \ctrl{1} & \qw      & \qw      & \targ{}  & \gate{\mathcal A_\gamma} & \qw \\
\lstick{$q_1$} & \gate{R_Y(x_1{+}\theta)} & \gate{R_Z(0.7x_1)} & \targ{}  & \ctrl{1} & \qw      & \qw      & \gate{\mathcal A_\gamma} & \qw \\
\lstick{$q_2$} & \gate{R_Y(x_2{+}\theta)} & \gate{R_Z(0.7x_2)} & \qw      & \targ{}  & \ctrl{1} & \qw      & \gate{\mathcal A_\gamma} & \qw \\
\lstick{$q_3$} & \gate{R_Y(x_3{+}\theta)} & \gate{R_Z(0.7x_3)} & \qw      & \qw      & \targ{}  & \ctrl{-3}& \gate{\mathcal A_\gamma} & \qw
\end{quantikz}
\caption{One data-re-uploading layer (\textsc{ring} entangling block,
amplitude damping $\mathcal A_\gamma$). The kernel uses $\rho(x)$ after $L$ such
layers via Eq.~\eqref{eq:kernel}.}
\label{fig:circuit}
\end{figure}
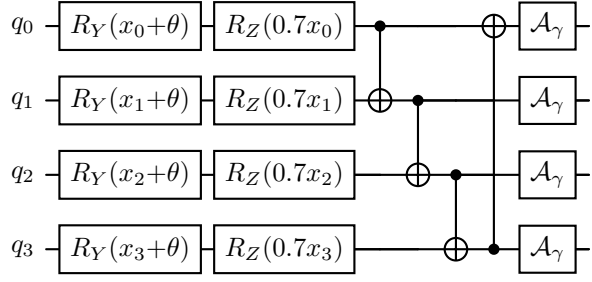

\paragraph{Trained-model architectures.} The trained ablations of
Sec.~\ref{sec:trained} (Table~\ref{tab:trained}) optimize the circuit parameters
and a linear classifier head end-to-end. The \emph{QViT-like} map (Fig.~\ref{fig:qvit})
re-uploads the input with trainable rotations and an all-to-all entangling block
per layer, reading out $\langle Z_i\rangle$ on every qubit. The \emph{QCNN-like}
map (Fig.~\ref{fig:qcnn}) replaces the entangler with two brick-pattern layers of
parametrized two-qubit \emph{convolutions} $U(\boldsymbol\theta)$. Amplitude
damping $\mathcal A_\gamma$ is injected during training in the noisy runs.

\begin{figure}[h]
\centering
\begin{quantikz}[column sep=6pt,row sep=8pt]
\lstick{$q_0$} & \gate{R_Y R_Z(\boldsymbol\theta_0)} & \gate[4]{\text{all-to-all CNOTs}} & \gate{\mathcal A_\gamma} & \meter{} \\
\lstick{$q_1$} & \gate{R_Y R_Z(\boldsymbol\theta_1)} &  & \gate{\mathcal A_\gamma} & \meter{} \\
\lstick{$q_2$} & \gate{R_Y R_Z(\boldsymbol\theta_2)} &  & \gate{\mathcal A_\gamma} & \meter{} \\
\lstick{$q_3$} & \gate{R_Y R_Z(\boldsymbol\theta_3)} &  & \gate{\mathcal A_\gamma} & \meter{}
\end{quantikz}
\caption{QViT-like trained map (shown on $4$ qubits): trainable angle encoding,
an all-to-all entangling block, optional injected $\mathcal A_\gamma$, repeated
$\times L$; the per-qubit $\langle Z_i\rangle$ feed a trained linear head.}
\label{fig:qvit}
\end{figure}
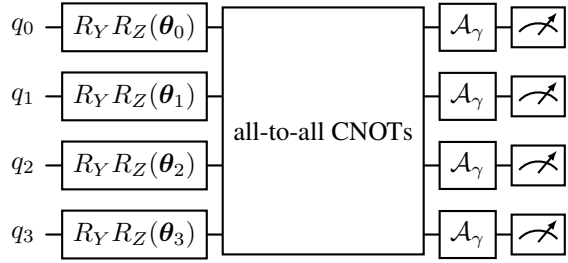

\begin{figure}[h]
\centering
\begin{quantikz}[column sep=7pt,row sep=7pt]
\lstick{$q_0$} & \gate{\text{enc}} & \gate[2]{U} & \qw         & \gate{\mathcal A_\gamma} & \meter{} \\
\lstick{$q_1$} & \gate{\text{enc}} &             & \gate[2]{U} & \gate{\mathcal A_\gamma} & \meter{} \\
\lstick{$q_2$} & \gate{\text{enc}} & \gate[2]{U} &             & \gate{\mathcal A_\gamma} & \meter{} \\
\lstick{$q_3$} & \gate{\text{enc}} &             & \gate[2]{U} & \gate{\mathcal A_\gamma} & \meter{} \\
\lstick{$q_4$} & \gate{\text{enc}} & \gate[2]{U} &             & \gate{\mathcal A_\gamma} & \meter{} \\
\lstick{$q_5$} & \gate{\text{enc}} &             & \qw         & \gate{\mathcal A_\gamma} & \meter{}
\end{quantikz}
\;\;
\begin{quantikz}[column sep=5pt,row sep=10pt]
\lstick{$a$} & \gate{R_Y} & \ctrl{1} & \gate{R_Z} & \targ{}  & \qw \\
\lstick{$b$} & \gate{R_Y} & \targ{}  & \gate{R_Z} & \ctrl{-1}& \qw
\end{quantikz}
\caption{QCNN-like trained map. \textbf{Left:} two brick-pattern layers of
two-qubit convolutions $U$ (here one layer on pairs $(0,1),(2,3),(4,5)$ and the
next on $(1,2),(3,4)$, with the ring closure $(5,0)$ omitted for the drawing),
optional $\mathcal A_\gamma$, then per-qubit readout into a linear head.
\textbf{Right:} the convolution block $U(\boldsymbol\theta)$.}
\label{fig:qcnn}
\end{figure}
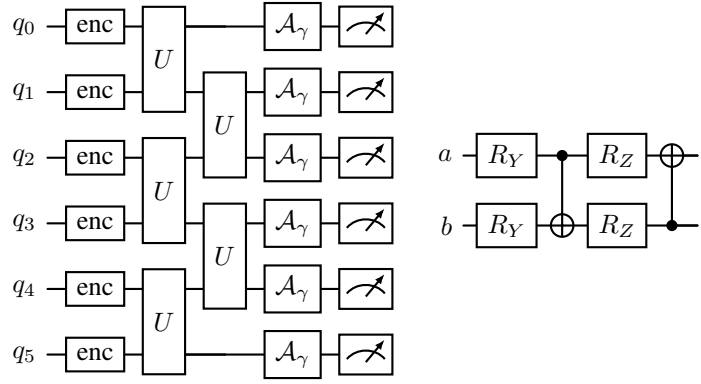

\section{Proofs}
\label{app:proofs}

Throughout, $\rho(x)\in\mathbb{C}^{D\times D}$ ($D=2^{n_q}$) are density matrices,
$\phi(x)=\mathrm{vec}\,\rho(x)\in\mathbb{C}^{D^2}$ are their vectorizations, and
$K\in\mathbb{R}^{n\times n}$ is the Gram matrix $K_{ij}=\operatorname{Tr}[\rho(x_i)\rho(x_j)]
=\langle\phi(x_i),\phi(x_j)\rangle$ with eigenvalues $\lambda_1\ge\cdots\ge\lambda_n\ge0$.

\subsection{Proof of Lemma~\ref{lem:psd}}
\emph{PSD.} For any $c\in\mathbb{R}^n$,
\[
c^\top K c=\sum_{i,j}c_i c_j\langle\phi(x_i),\phi(x_j)\rangle
=\Big\langle \sum_i c_i\phi(x_i),\,\sum_j c_j\phi(x_j)\Big\rangle
=\Big\|\sum_i c_i\phi(x_i)\Big\|^2\ge0,
\]
where we used the (real part of the) Hermitian HS inner product; since
$\rho(x_i)$ are Hermitian, $K_{ij}=\operatorname{Tr}[\rho(x_i)\rho(x_j)]$ is real
and symmetric. Hence $K\succeq0$ and all $\lambda_i\ge0$, so $\deff(K)$ in
\eqref{eq:deff} is well defined.

\emph{Range.} Writing $\Lambda=\sum_i\lambda_i$ and $Q=\sum_i\lambda_i^2$, the
Cauchy--Schwarz inequality $\Lambda^2=(\sum_i\lambda_i\cdot1)^2\le n\sum_i\lambda_i^2=nQ$
gives $\deff=\Lambda^2/Q\le n$, with equality iff all $\lambda_i$ are equal (flat
spectrum). Since $\lambda_i\ge0$, $\Lambda^2=\sum_i\lambda_i^2+\sum_{i\ne j}\lambda_i\lambda_j\ge Q$,
so $\deff\ge1$, with equality iff exactly one $\lambda_i>0$, i.e.\ $\operatorname{rank}K=1$.
\hfill$\square$

\subsection{Proof of Proposition~\ref{prop:noise}}
\emph{Exact decomposition.} For the global depolarizing channel
$\mathcal N_p[\rho]=(1-p)\rho+p\,I/D$ and $\rho_p(x)=\mathcal N_p[\rho(x)]$,
bilinearity of the trace gives
\begin{align*}
(K_p)_{ij}&=\operatorname{Tr}\big[\rho_p(x_i)\rho_p(x_j)\big]\\
&=(1-p)^2\operatorname{Tr}[\rho(x_i)\rho(x_j)]
+\tfrac{(1-p)p}{D}\operatorname{Tr}[\rho(x_i)]
+\tfrac{(1-p)p}{D}\operatorname{Tr}[\rho(x_j)]
+\tfrac{p^2}{D^2}\operatorname{Tr}[I]\\
&=(1-p)^2K_{ij}+\tfrac{(1-p)p}{D}+\tfrac{(1-p)p}{D}+\tfrac{p^2}{D},
\end{align*}
using $\operatorname{Tr}\rho=1$ and $\operatorname{Tr}I=D$. The last three terms
sum to $\frac{p}{D}\big(2(1-p)+p\big)=\frac{p(2-p)}{D}$, independent of $i,j$,
which proves
\[
K_p=(1-p)^2K+\tfrac{p(2-p)}{D}\,\mathbf 1\mathbf 1^\top .
\tag{\ref{eq:Kp}}
\]

\emph{Limit.} As $p\to1$, $(1-p)^2\to0$ and $\frac{p(2-p)}{D}\to\frac1D$, so
$K_p\to\frac1D\mathbf1\mathbf1^\top$, which has rank $1$; by Lemma~\ref{lem:psd},
$\deff(K_p)\to1$.

\emph{Monotone concentration under constant row sums.} Suppose $K\mathbf1=c\mathbf1$, i.e.\ $\mathbf1$ is an eigenvector of $K$ (note: Perron--Frobenius only makes the top eigenvector \emph{positive}, not constant, so this is a genuine extra assumption). Then $K$ and $\mathbf1\mathbf1^\top$
are simultaneously diagonalizable: with orthonormal eigenvectors $u_1=\mathbf1/\sqrt n,
u_2,\dots,u_n$, we have $K=\sum_k\lambda_k u_ku_k^\top$ and
$\mathbf1\mathbf1^\top=n\,u_1u_1^\top$. Put $s=(1-p)^2$ and note
$\frac{p(2-p)}{D}\,n=\frac{n}{D}(1-s)=:\,a(s)$ with $\beta:=n/D$, $a=\beta(1-s)$.
The eigenvalues of $K_p$ are
\[
\mu_1(p)=s\lambda_1+a,\qquad \mu_k(p)=s\lambda_k\ \ (k\ge2).
\]
Consider the top-eigenvalue share $\sigma_1=\mu_1/\sum_k\mu_k
=\dfrac{s\lambda_1+\beta(1-s)}{s\Lambda+\beta(1-s)}=\dfrac{\beta+s(\lambda_1-\beta)}{\beta+s(\Lambda-\beta)}$.
Differentiating in $s$,
\[
\frac{d\sigma_1}{ds}
=\frac{(\lambda_1-\beta)\big(\beta+s(\Lambda-\beta)\big)-(\Lambda-\beta)\big(\beta+s(\lambda_1-\beta)\big)}
{\big(\beta+s(\Lambda-\beta)\big)^2}
=\frac{\beta(\lambda_1-\Lambda)}{\big(\beta+s(\Lambda-\beta)\big)^2}\le0,
\]
since $\Lambda=\sum_k\lambda_k\ge\lambda_1$. Thus $\sigma_1$ is non-increasing in
$s$, i.e.\ non-decreasing in $p$, and $\deff(K_p)=1/\sum_k(\mu_k/\sum_l\mu_l)^2$
decreases; this is rigorous under $K\mathbf1=c\mathbf1$.

\emph{General case and the obstruction.} Without constant row sums, $K$ and
$\mathbf1\mathbf1^\top$ do not commute, the eigenvalues of $K_p$ are not the
$\mu_k$ above, and the simultaneous-diagonalization step fails. The exact
participation ratio is nevertheless \eqref{eq:deffp}, obtained from
$\operatorname{tr}K_p=sT+\tfrac{p(2-p)}{D}n$ and
$\operatorname{tr}(K_p^2)=s^2Q+2s\tfrac{p(2-p)}{D}S+\big(\tfrac{p(2-p)}{D}\big)^2n^2$
using $\operatorname{tr}(\mathbf1\mathbf1^\top)=n$,
$\operatorname{tr}(K\mathbf1\mathbf1^\top)=\mathbf1^\top K\mathbf1=S$, and
$\operatorname{tr}((\mathbf1\mathbf1^\top)^2)=n^2$. This is a ratio of quadratics
in $s$; its derivative's sign depends on $S=\mathbf1^\top K\mathbf1$, the
cross-term that the constant-row-sum condition would fix, so global monotonicity
is \emph{not} guaranteed without structure. We therefore claim only: (i) the exact
limit $\deff(K_p)\to1$; (ii) the rigorous upper bound $\deff(K_p)\le1/\sigma_1(p)^2$
with $\sigma_1$ the top-eigenvalue share, which tends to $1$; (iii) provable
monotonicity under $K\mathbf1=c\mathbf1$; and (iv) the empirical fact that
$\deff(K_p)$ is strictly decreasing---hence injective---in every experiment,
which is all the collapse of Prop.~\ref{prop:align}(i) requires. \hfill$\square$

\subsection{Proof of Proposition~\ref{prop:ad} (amplitude damping)}
The single-qubit amplitude-damping channel has Kraus operators
$E_0=\mathrm{diag}(1,\sqrt{1-\gamma})$ and $E_1=\sqrt{\gamma}\,\ket0\!\bra1$, so on
$\rho=\begin{psmallmatrix}\rho_{00}&\rho_{01}\\\rho_{10}&\rho_{11}\end{psmallmatrix}$
it acts as $\rho_{00}\mapsto\rho_{00}+\gamma\rho_{11}$,
$\rho_{11}\mapsto(1-\gamma)\rho_{11}$, $\rho_{01}\mapsto\sqrt{1-\gamma}\,\rho_{01}$.

\emph{(a) Limit.} As $\gamma\to1$ every qubit's state tends to $\ket0\!\bra0$
(populations flow to $\ket0$, coherences vanish), so
$\rho_\gamma(x)\to\ket0\!\bra0^{\otimes n_q}$ uniformly in $x$; hence
$(K_\gamma)_{ij}=\operatorname{Tr}[\rho_\gamma(x_i)\rho_\gamma(x_j)]\to1$ for all
$i,j$, i.e.\ $K_\gamma\to\mathbf1\mathbf1^\top$, which is rank one, so
$\deff(K_\gamma)\to1$ by Lemma~\ref{lem:psd}.

\emph{(b) Single-qubit contraction.} In Bloch coordinates
$\rho=\tfrac12(I+X\sigma_x+Y\sigma_y+Z\sigma_z)$ the rules above give
$X'=\sqrt{1-\gamma}\,X$, $Y'=\sqrt{1-\gamma}\,Y$, and, using
$\rho_{11}=(1-Z)/2$, $Z'=Z+\gamma(1-Z)=(1-\gamma)Z+\gamma$. Since
$\|\rho-\sigma\|_{\mathrm{HS}}^2=\tfrac12(\Delta X^2+\Delta Y^2+\Delta Z^2)$,
\[
\|\mathcal A_\gamma\rho-\mathcal A_\gamma\sigma\|_{\mathrm{HS}}^2
=\tfrac12\big[(1-\gamma)\Delta X^2+(1-\gamma)\Delta Y^2+(1-\gamma)^2\Delta Z^2\big]
\le(1-\gamma)\,\|\rho-\sigma\|_{\mathrm{HS}}^2,
\]
because $(1-\gamma)^2\le(1-\gamma)$. Thus $\mathcal A_\gamma$ strictly contracts
the traceless part by the factor $\sqrt{1-\gamma}$. For a product feature map
$\rho(x)=\bigotimes_q\rho_q(x)$ the kernel factorizes,
$K(x,x')=\prod_q\operatorname{Tr}[\rho_q(x)\rho_q(x')]$, and each factor's
distinguishability contracts, so every off-diagonal $K_\gamma$ entry increases
toward $1$ while the diagonal stays $\le1$; the centered Gram matrix shrinks in
Frobenius norm and $\deff(K_\gamma)$ is non-increasing.

\emph{(c) Boundary.} On a single qubit $\mathcal A_\gamma(I)=I+\gamma Z$, so in the
Pauli basis the $(I,Z)$ block is $\begin{psmallmatrix}1&0\\\gamma&1-\gamma\end{psmallmatrix}$,
which is non-normal with largest singular value $>1$ for $\gamma\in(0,1)$. Hence
$\mathcal A_\gamma^{\otimes n_q}$ can \emph{increase} the Hilbert--Schmidt norm of
operators with a nonzero trace component; the difference of two \emph{entangled}
feature states can have weight along such directions, so the clean global
contraction available for depolarizing noise (Prop.~\ref{prop:noise}) is not
guaranteed in general. Under the constant-row-sum condition---$\mathbf1$ remaining an
eigenvector of $K_\gamma$, which holds for strictly positive kernels---the
top-share argument of Prop.~\ref{prop:noise} again gives monotone $\deff$
contraction; empirically (Tables~\ref{tab:noise},~\ref{tab:sweet}) $\deff(K_\gamma)$
decreases strictly at every step. \hfill$\square$

\subsection{Proof of Proposition~\ref{prop:cap}}
We use the standard capacity control of kernel ridge regression (KRR) by the
effective dimension $d_\gamma(K)=\operatorname{tr}[K(K+\gamma I)^{-1}]=\sum_i\lambda_i/(\lambda_i+\gamma)$
\citep{abbas2021power}. For a bounded-norm RKHS ball, the local Rademacher
complexity at scale $\gamma$ obeys
\[
\mathcal R_n\;\lesssim\;\sqrt{\tfrac1n\textstyle\sum_i\min(\lambda_i,\gamma)}\;\le\;\sqrt{\tfrac{2\gamma}{n}\,d_\gamma(K)} .
\]
The second inequality uses the elementary bound
$\min(\lambda,\gamma)\le \dfrac{2\gamma\lambda}{\lambda+\gamma}=2\gamma\cdot\dfrac{\lambda}{\lambda+\gamma}$,
valid for all $\lambda,\gamma\ge0$: if $\lambda\le\gamma$ the right side is
$\ge 2\gamma\lambda/(2\gamma)=\lambda$, and if $\lambda>\gamma$ it is
$\ge 2\gamma\lambda/(2\lambda)=\gamma$; summing over $i$ gives
$\sum_i\min(\lambda_i,\gamma)\le 2\gamma\,d_\gamma(K)$. Standard local Rademacher
bounds then give an excess-risk gap $\tilde O\!\big(\sqrt{d_\gamma(K)/n}\big)$.
Each summand of $d_\gamma$ is increasing in its eigenvalue,
\[
\frac{\partial}{\partial\lambda_i}\frac{\lambda_i}{\lambda_i+\gamma}=\frac{\gamma}{(\lambda_i+\gamma)^2}>0,
\]
so any spectral contraction that lowers eigenvalues---in particular
$K\mapsto K_p$ of Proposition~\ref{prop:noise}, which multiplies the non-top
eigenvalues by $(1-p)^2$---decreases $d_\gamma$ and tightens the bound. When the
empirical risk is $\approx0$ (the overfitting regime induced by label noise), the
excess risk is gap-dominated, so reducing $\deff$ reduces test risk; once the
top, signal-carrying eigen-directions are themselves attenuated the bias grows and
the risk turns up, yielding an interior optimum. We use the participation ratio
$\deff$ \eqref{eq:deff} as a $\gamma$-free surrogate for $d_\gamma$; both are
non-decreasing under simultaneous inflation of the spectrum and co-vary on the
spectra we observe. \hfill$\square$

\subsection{Proof of Proposition~\ref{prop:align}}
Let the target function have coefficients $\beta_i=\langle f^\star,u_i\rangle$ in
the kernel eigenbasis $\{u_i\}$ (eigenvalues $\lambda_i$). The KRR estimator with
ridge $\gamma$ and $n$ samples (label noise variance $\tau^2$) has the classical
bias--variance excess-risk decomposition
\[
\mathcal E(\hat f)=\underbrace{\sum_i\Big(\frac{\gamma}{\lambda_i+\gamma}\Big)^2\beta_i^2}_{\text{bias}^2\ (\text{alignment})}
\;+\;\underbrace{\frac{\tau^2}{n}\sum_i\Big(\frac{\lambda_i}{\lambda_i+\gamma}\Big)^2}_{\text{variance}\ (\text{capacity})}.
\]
The variance term depends on $K$ only through its spectrum $\{\lambda_i\}$, and
the bias term depends additionally on the target energy $\{\beta_i^2\}$ across
eigen-directions (the information in $\mathrm A(K)$).

\emph{Claim (i): exact one-parameter collapse.} The depolarizing family $K_p$ of
Prop.~\ref{prop:noise} is determined by the single scalar $p$ (the whole kernel,
hence its spectrum and the $\beta_i$, are explicit functions of $p$ via
\eqref{eq:Kp}). Therefore the excess risk $\mathcal E$ is a function of $p$. If, in
addition, $\deff(K_p)$ is \emph{injective} in $p$---which we do not claim in
general but verify holds (indeed $\deff$ is strictly monotone) in every experiment,
and which holds provably under the constant-row-sum condition of
Prop.~\ref{prop:noise}---then $\mathcal E$ is a function of $\deff$ and accuracy
collapses exactly onto $\mathrm{acc}=f(\deff)$. The collapse along the noise family
thus requires only injectivity of $p\mapsto\deff(K_p)$, not the eigenvector-fixing
of the (false) Perron reading.

\emph{Claim (ii): cross-family collapse is conditional.} For two kernels of
different spectral \emph{shape}, equal $\deff=(\operatorname{tr}K)^2/\operatorname{tr}(K^2)$
constrains only this one scalar; the variance term
$\tfrac{\tau^2}{n}\sum_i(\lambda_i/(\lambda_i+\gamma))^2$ and the bias term are not
determined by it, so equal $\deff$ need not give equal risk. Collapse onto a
common curve therefore requires, beyond equal $\deff$, comparable spectral shape
\emph{and} comparable alignment $\mathrm A(K)$. The unentangled product map
factorizes, $k(x,x')=\prod_q\operatorname{Tr}[\rho_q(x)\rho_q(x')]$, cannot
represent cross-feature correlations, and has measurably lower alignment
(Sec.~\ref{sec:unify}); its larger bias term places it below the entangled curve
at every $\deff$. The entangled ansatze, having comparable shape and alignment,
collapse empirically. \hfill$\square$

\section{Experimental details}
\label{app:exp}

\paragraph{Software and simulators.} All quantum computations use PennyLane with
the \texttt{default.mixed} density-matrix simulator (so amplitude-damping and
depolarizing channels are exact, not sampled) for noisy kernels and
\texttt{default.qubit} for noiseless/trained statevector runs. Kernels are
assembled as a single dense product over flattened density matrices,
$K=\mathrm{Re}(V V^{\dagger})$ with $V_i=\mathrm{vec}\,\rho(x_i)$. Classifiers use
scikit-learn \texttt{SVC} with a precomputed kernel. Real-hardware runs use
\texttt{qiskit-ibm-runtime} against IBM Heron devices.

\paragraph{Feature map.} Inputs are standardized PCA features scaled to
$[-2.7,2.7]$. The $n_q$-qubit, $L$-layer data-re-uploading circuit applies, in
each layer $\ell$, single-qubit rotations $R_Y(x_i+\theta_{\ell,i})R_Z(0.7\,x_i)$
on every qubit $i$ followed by an entangling block of CNOTs; $\theta_{\ell,i}$ are
fixed i.i.d.\ $\mathcal N(0,0.5^2)$ (the kernel is data-driven, not trained). The
entangling block is \textsc{product} (none), \textsc{chain} $i\!\to\!i{+}1$,
\textsc{ring} $i\!\to\!i{+}1\bmod n_q$, or \textsc{all-to-all}. Amplitude damping
of rate $p$ is applied after each entangling block, with $n_q{=}6$ and $L{=}2$ in
the main experiments (varied as described above).

\paragraph{Datasets and overfitting protocol.} \textsc{Digits} ($8{\times}8$,
sklearn), Fashion-MNIST and BloodMNIST (MedMNIST), each PCA-reduced to $6$ (or $8$)
components. A controlled overfitting regime is induced by replacing a fraction
$\ell$ of \emph{training} labels with uniform random labels; the test set is
always clean. Unless stated otherwise, $\ell{=}25\%$, $n_{\text{train}}{=}150$,
$n_{\text{test}}{=}250$ (kernel experiments).

\paragraph{Seeds and error bars.} Multi-seed results re-draw, per seed, the PCA
fit, the train/test split, the label corruption mask, and the random circuit
parameters $\theta$; we report mean$\pm$std over $5$ seeds. Spectral descriptors
($\deff$, $H_{\mathrm{spec}}$, $\mathrm A(K)$) are computed on the training block.

\paragraph{Per-experiment configuration.} Table~\ref{tab:configs} lists the
settings of each experiment.

\begin{table}[h]
\centering
\caption{Configuration of the experiments.}
\label{tab:configs}
\small
\begin{tabular}{lll}
\toprule
Experiment & Setting & Section \\
\midrule
Entanglement (Table~\ref{tab:ent}) & $4$ ansatze, $p{=}0$, $5$ seeds & \ref{sec:ent}\\
Noise sweep (Table~\ref{tab:noise}) & ring/all-to-all, $p\in[0,0.45]$ & \ref{sec:noise}\\
Sweet spot (Table~\ref{tab:sweet}) & all-to-all, $\ell\in\{0,10,25,40\}\%$, $p\le0.9$ & \ref{sec:sweetspot}\\
Collapse / alignment (Table~\ref{tab:codesign}) & $4{\times}6$ grid, $\mathrm A(K)$, $5$ seeds & \ref{sec:unify}\\
Datasets (Table~\ref{tab:datasets}) & Digits/Fashion/Blood, $4{\times}5$ grid & \ref{sec:robust}\\
Depth/width & $L\in\{1{-}4\}$, $n_q\in\{6,8\}$ & \ref{sec:robust}\\
Sign flip (Table~\ref{tab:signflip}) & overfit vs.\ $3$-qubit product clean & \ref{sec:signflip}\\
Label-free selection & clean-val and $\deff\!\approx\!\sqrt n$ rules & \ref{sec:labelfree}\\
Spectral baselines & noise vs.\ $C$-sweep vs.\ $\lambda^\alpha$ shrinkage & \ref{sec:classical}\\
Trained QViT-/QCNN-like & autograd, per-epoch $\deff$, noise on/off & \ref{sec:trained}\\
IBM Heron & feature $\langle Z_i\rangle$, hardware vs.\ ideal & \ref{sec:robust}\\
\bottomrule
\end{tabular}
\end{table}

\paragraph{Trained models.} The trained ablations optimize the circuit parameters
and a linear head end-to-end with PennyLane autograd (Adam, lr $0.05$); injected
noise during training uses \texttt{default.mixed}. The QCNN-style model uses
two brick-pattern layers of parametrized $2$-qubit convolutions with pooling
($6\!\to\!3\!\to\!2$ qubits) and reads out $\langle Z\rangle$ on the survivors.

\paragraph{Real-hardware run.} We run on the real Heron device
\texttt{ibm\_kawasaki} (accessed through the dedicated on-prem instance, which
executes immediately, unlike the shared premium queues). We transpile the
all-to-all feature circuit (deepened to $L{=}5$ so accumulated gate noise is
non-negligible) to the native basis (\texttt{sx, rz, cz, x}) and estimate the
$\langle Z_i\rangle$ and $\langle Z_iZ_{i+1}\rangle$ features of $70$ inputs with
$4096$ shots via the Estimator primitive, comparing the hardware feature vectors
and their induced $\deff$ to the ideal (statevector) values.

\end{document}

%% file: math_commands.tex
\usepackage{amsmath,amsfonts,bm}

\def\eqref#1{equation~\ref{#1}}

\def\1{\bm{1}}

\DeclareMathAlphabet{\mathsfit}{\encodingdefault}{\sfdefault}{m}{sl}
\SetMathAlphabet{\mathsfit}{bold}{\encodingdefault}{\sfdefault}{bx}{n}

